\UseRawInputEncoding
\pdfoutput=1
\documentclass[11pt]{article}
\usepackage{natbib}
\usepackage[final]{acl}
\usepackage{tcolorbox}
\tcbuselibrary{skins, breakable, theorems}
\usepackage[frozencache,cachedir=minted-cache]{minted}
\usepackage{times}
\usepackage{latexsym}
\usepackage{algorithm,algpseudocode}
\usepackage{amsmath} % 引入 amsmath 宏包，提供数学环境和符号支持
\usepackage{amsfonts} % 引入 amsfonts 宏包，提供更多的数学字体
\usepackage[T1]{fontenc}
\usepackage{amsthm}
\usepackage{booktabs}
\usepackage{pifont}

\usepackage[utf8]{inputenc}

\usepackage{microtype}

\usepackage{inconsolata}

\usepackage{graphicx}
\usepackage{subcaption}

\newtheorem{theorem}{Theorem}[section]
\newtheorem{lemma}[theorem]{Lemma}

\theoremstyle{definition}

\theoremstyle{remark}

\title{Triviality Corrected Endogenous Reward}

\author{
\textbf{Xinda Wang}$^{1}$\thanks{Contributed equally.} \quad
\textbf{Zhengxu Hou}$^{2}$\footnotemark[1] \quad
\textbf{Yangshijie Zhang}$^{3}$ \quad
\textbf{Bingren Yan}$^{2}$ \quad
\textbf{Jialin Liu}$^{1}$ \\
\textbf{Chenzhuo Zhao}$^{1}$ \quad
\textbf{Zhibo Yang}$^{2}$ \quad
\textbf{Bin-Bin Yang}$^{2}$ \quad
\textbf{Feng Xiao}$^{2}$\thanks{Corresponding author.} \\
$^{1}$Peking University \quad
$^{2}$Alibaba Group \quad
$^{3}$Lanzhou University
}

\begin{document}
\maketitle
~\begin{abstract}
Reinforcement learning for open-ended text generation is constrained by the lack of verifiable rewards, necessitating reliance on judge models that require either annotated data or powerful closed-source models. Inspired by recent work on unsupervised reinforcement learning for mathematical reasoning using confidence-based endogenous rewards, we investigate whether this principle can be adapted to open-ended writing tasks. We find that directly applying confidence rewards leads to \textit{Triviality Bias}: the policy collapses toward high-probability outputs, reducing diversity and meaningful content. We propose TCER (Triviality Corrected Endogenous Reward), which addresses this bias by rewarding the relative information gain between a specialist policy and a generalist reference policy, modulated by a probability-dependent correction mechanism. Across multiple writing benchmarks and model architectures, TCER achieves consistent improvements without external supervision. Furthermore, TCER also transfers effectively to mathematical reasoning, validating the generality of our approach across different generation tasks.

\end{abstract}

% modulated by a probability-dependent correction mechanism 这里写的比较奇怪，但没想到好的写法

~\begin{figure}[ht]
    \centering
    \includegraphics[width=0.48\textwidth]{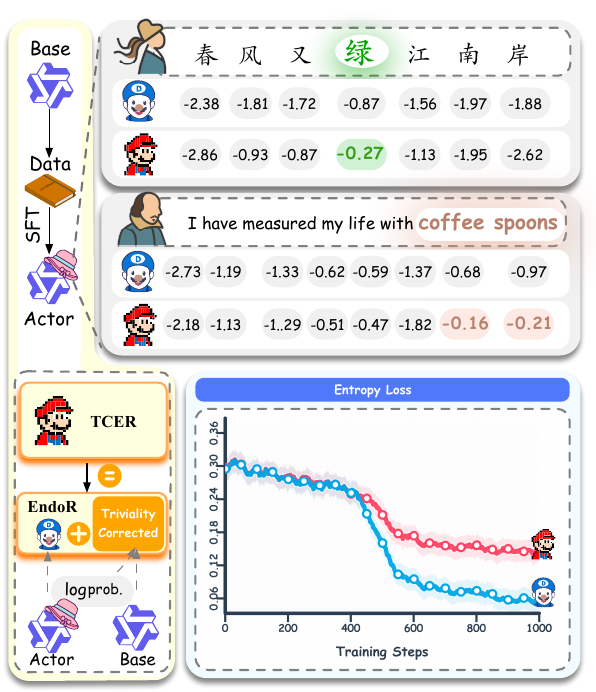}
    \caption{TCER computes rewards using both base and actor models, emphasizing high-entropy meaningful content over EndoR. After RL Training, TCER shows gentler entropy decay with enhanced exploration.}
    \label{fig:Introduction}
    \vspace{-4mm}
\end{figure}

\section{Introduction}

The capabilities of Large Language Models (LLMs) have advanced significantly, with alignment techniques playing a pivotal role in their success \cite{kumar2025llm, tie2025survey}. Following supervised fine-tuning (SFT), Reinforcement Fine-Tuning (RFT) has emerged as a crucial step for further optimizing model behavior according to specific objectives \cite{wu2025generalization, jin2025rl}. While RFT has proven effective in domains with clear evaluation criteria, its application to writing presents challenges that remain unresolved \cite{gooding2025writing, li2025jointly}. 

Unlike mathematical reasoning or code generation where correctness is objectively verifiable, writing quality is subjective and multifaceted \cite{ying2025beyond}. Current approaches rely on LLMs as judges \cite{zheng2023judging}, facing a dilemma: powerful proprietary models like GPT-4 provide high-quality evaluation \cite{wu2025longwriter} but incur substantial API costs \cite{huang2025empirical}, while training custom judge models requires extensive human-annotated preference data \cite{feng2025we, lu2025writing, wang2025evolvr}. The latter approach is further complicated by the subjectivity of writing, which makes consistent annotation challenging with different annotators often disagreeing on quality metrics \cite{ni2025can, chiang2023can}. In parallel, the mathematical reasoning community has developed a solution to avoid external supervision \cite{li2025generalist}: \textit{endogenous rewards}, where models use their own confidence scores (log-probabilities) as reward signals. This approach has demonstrated success in mathematics \cite{prabhudesai2025maximizing}. The appeal of such judge-free methods \cite{li2025confidence} naturally raises the question: can this paradigm be adapted to creative writing, solving the reward model problem?

In this paper, we perform a systematic investigation of endogenous reward RFT for writing. As shown in Figure\ref{fig:Introduction}, we find that endogenous rewards can improve over SFT baselines, suggesting that confidence signals remain useful in open-ended generation. However, we also identify a degeneration phenomenon: optimizing for self-confidence tends to drive the policy toward low-entropy behavior and reducing diversity and content richness \cite{li2025jointly}. We refer to this failure mode as \textit{Triviality Bias}.

To address this \textit{Triviality Bias}, we propose TCER (Triviality Corrected Endogenous Reward). We define a generalist reference policy $\pi_b$ and a specialist policy $\pi_s$ obtained by finetuning $\pi_b$ on high-quality in-domain data. TCER define the token-level log-likelihood ratio $\log\frac{\pi_s}{\pi_b}$, which favors tokens that are characteristic of the specialist's domain. To counteract the bias towards high-confidence predictions, we modulate this reward with an adaptive weight of $(1-\pi_s)^\lambda$. This mechanism suppresses the incentive for high-probability tokens and allocates reward to choices that are informative yet less predictable, thereby promoting exploration. We optimize using GRPO with reference augmentation to align model outputs with high-quality targets. Our main contributions are:
\begin{enumerate}
    \item We provide a systematic study of judge-free RFT via confidence-based endogenous rewards for writing, showing both its effectiveness and a degeneration phenomenon (\textit{Triviality Bias}) characterized by low-entropy output.
    \item We propose \textsc{TCER}, which mitigates triviality bias by rewarding relative information gain between a specialist policy and a generalist reference policy, modulated by a probability-dependent correction.
    \item We conduct extensive experiments across multiple models and datasets in writing and mathematical reasoning, demonstrating consistent improvements over baselines and cross-domain generalization.
\end{enumerate}
\vspace{-9pt}
~\section{Related Work}
\subsection{Open-Ended Text Generation}
Training models for proficient open-ended writing has been a central challenge in LLM development \cite{wei2025igniting}. The dominant approach relies on constructing high-quality datasets for SFT \cite{wang2024weaver}. LongWriter addresses the challenge of generating coherent long-form content by extending context windows and introducing specialized data collection pipelines \cite{bai2024longwriter}. DeepWriting takes a different approach through reverse-engineered reasoning \cite{wang2025reverse}, where high-quality outputs are augmented with synthetic planning traces. These datasets enable SFT models to achieve reasonable writing capabilitie \cite{kim2025align}, establishing the foundation for further optimization.

Beyond SFT, RFT has emerged as a promising direction for enhancing writing quality \cite{huang2025blending}. However, existing RL approaches for writing predominantly rely on external reward models. LongWriter-zero employs LLMs as judges\cite{wu2025longwriter}, using GPT-4 or similar models to score generated texts during RL training. Similarly, Writing-zero\cite{lu2025writing} and EvolvR\cite{wang2025evolvr} utilize pairwise preference judgments from large language models to construct reward signals for story generation and creative writing tasks. While these methods demonstrate improvements over SFT baselines, they inherit the fundamental limitation of requiring costly external evaluation\cite{wu2025rlac}, whether from human annotators or proprietary models.

\begin{figure*}[ht]
    \centering
    \includegraphics[width=0.95\textwidth]{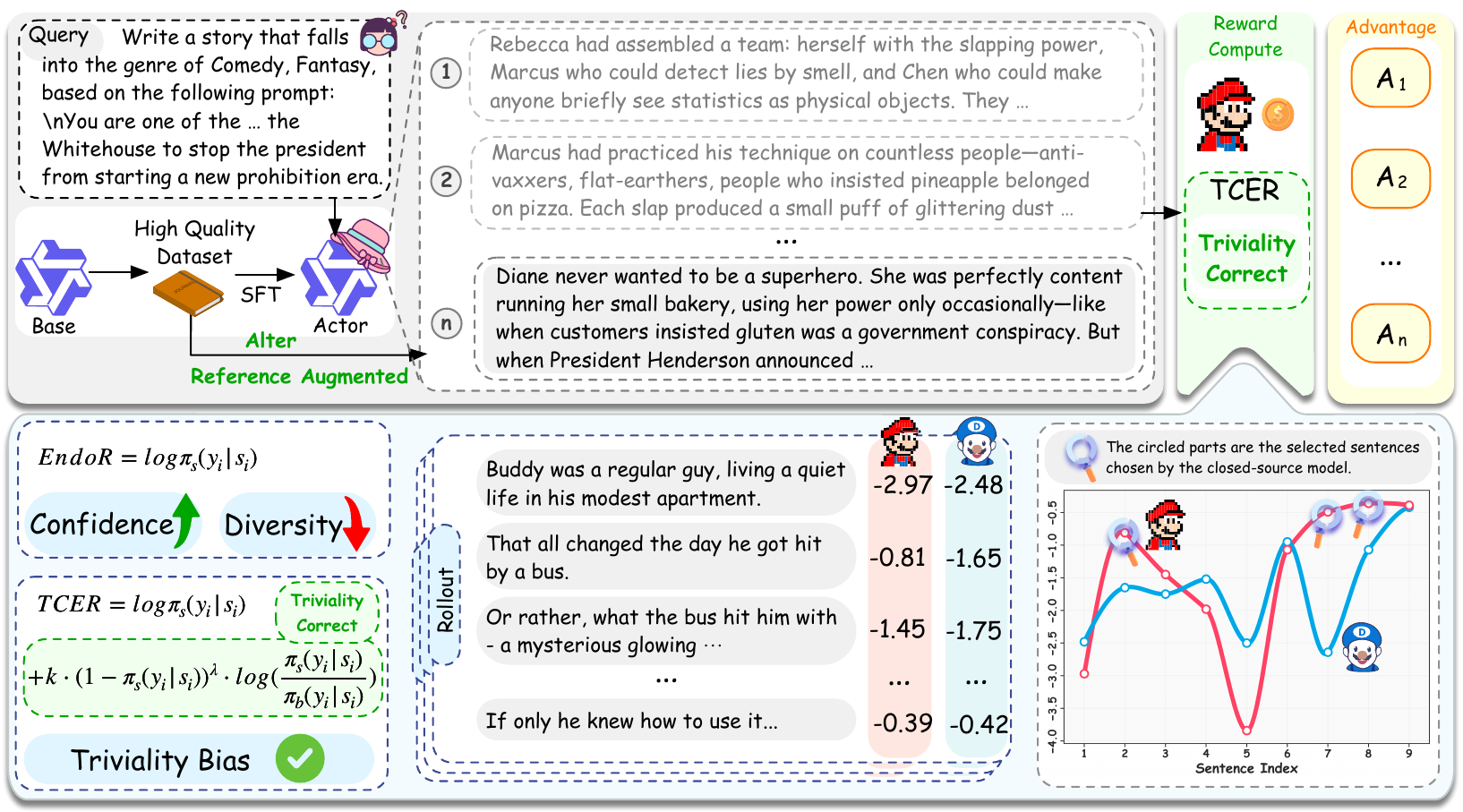}
    \caption{Overview of TCER training pipeline and reward comparison. (a) Training workflow: SFT on high-quality data to obtain $\pi_s$, followed by Reference Augmented GRPO with TCER that incorporate information gain $\phi$ and gating weight $w$. (b) Sentence-level reward visualization on a rollout: token-level rewards are averaged within each sentence comparing TCER and EndoR. (Rewards are average log-probs; higher is closer to $0$.)}
    \label{fig:method}    
    \vspace{-4mm}
\end{figure*}

\subsection{Unsupervised Rewards for RFT}

The mathematical reasoning domain has pioneered the exploration of unsupervised reward \cite{zhao2025learning} mechanisms for RL training. TTRL\cite{zuo2025ttrl} introduces entropy-based rewards by clustering rollout answers and rewarding based on cluster sizes, where higher consistency in solution clusters indicates more reliable reasoning paths \cite{zhang2025right}. Another line of work explores confidence-based approaches \cite{damani2025beyond，xiong2025self, prabhudesai2025maximizing}. Various methods directly leverage the model's inherent confidence as direct rewards for self-improvement \cite{li2025confidence}, and some constructing process reward models from confidence scores to provide fine-grained feedback \cite{tan2025gtpo}. Recently, theoretical work has established the endogenous reward hypothesis \cite{li2025generalist}, demonstrating that sufficiently capable models implicitly encode reward functions within their parameters. This provides mathematical justification for these confidence-based methods.

However, these methods have remained confined to mathematical reasoning and have not been applied to open-ended writing task \cite{yue2025does}. The challenge is the risk of entropy collapse \cite{zhang2025noveltybench, cui2025entropy, wang2025beyond} that confidence-based approaches inherently drive models toward generating high-frequency, formulaic content, resulting in decreased diversity and creative expression \cite{li2025jointly, west2025base}. Our work addresses this challenge by introducing a correction mechanism that mitigates entropy collapse in writing tasks, enabling the application of unsupervised RL methods beyond mathematical reasoning.

~\section{Methodology}
This section formulates judge-free RFT for open-ended text generation. We analyze the \textit{triviality bias} induced by endogenous rewards. We derive TCER from an information theoretic decomposition of a generalist policy and an in-domain specialist policy, modulated by a probability-dependent correction. Finally, we describe optimization with Reference Augmented Group Relative Policy Optimization (GRPO) \cite{shao2024deepseekmath}. The framework is shown in Figure\ref{fig:method}. Additional derivations are deferred to Appendix~\ref{app:derivations}.

\subsection{Implicit Rewards and Triviality Bias}
Recently, connections between next-token prediction and inverse reinforcement learning (IRL) motivate using model confidence as an endogenous reward. In IRL, an optimal policy under reward $r$ takes a Boltzmann form:
\begin{equation}
    \pi^*(y\mid x) \propto \exp\left(\frac{Q^*(x,y)}{\alpha}\right),
\end{equation}
where $Q^*$ is the soft $Q$-function and $\alpha$ is the temperature. In our setting, let $\pi_s$ denote a specialist policy obtained by SFT on high-quality in-domain data. A confidence based endogenous reward assigns each generated token $y_i$ in context $s_i$ the reward
\begin{equation}
    r_e(y_i \mid s_i) = \log \pi_s(y_i \mid s_i),
\end{equation}
where $s_i$ denotes the decoding context at position $i$. The degeneration induced by $r_e$ can be characterized through a context maximization. For any context $s$, consider the one step objective
\begin{equation}
    \max_{\pi(\cdot\mid s)} \ \mathbb{E}_{y\sim \pi(\cdot\mid s)}\big[\log \pi_s(y\mid s)\big].
    \label{eq:one_step_obj}
\end{equation}
Because the expectation in Eq.~\eqref{eq:one_step_obj} is linear in $\pi(\cdot\mid s)$, the maximizer is attained at an extreme point of the simplex which is a deterministic distribution:     
\begin{equation}
    \pi^{*}(\cdot\mid s)=\delta\!\left(y=\arg\max_{v\in\mathcal{V}}\pi_s(v\mid s)\right).
\end{equation}
Optimizing $r_e$ encourages low-entropy behavior by concentrating on the high probability tokens under $\pi_s$. In open-ended writing, these tokens correspond to high-frequency, yielding \textit{triviality bias}: generations become increasingly templated, with reduced diversity.

\subsection{Information Theoretic Decomposition}
To isolate token-level signals of in-domain specialization beyond generic preference, we decompose the specialist policy $\pi_s$ relative to a generalist baseline $\pi_b$, where $\pi_b$ and $\pi_s$ are the same model before and after in-domain fine-tuning, respectively. For any token and context,
\begin{equation}
    \pi_s(y_i\mid s_i) = \pi_b(y_i\mid s_i) \cdot \frac{\pi_s(y_i\mid s_i)}{\pi_b(y_i\mid s_i)}.
\end{equation}
Taking logarithms yields
\begin{equation}
    \log \pi_s(y_i\mid s_i) = \log \pi_b(y_i\mid s_i) + \phi(y_i\mid s_i),
\end{equation}
where we define the specific information gain as the token-level log-likelihood ratio
\begin{equation}
    \phi(y_i\mid s_i) = \log \frac{\pi_s(y_i \mid s_i)}{\pi_b(y_i \mid s_i)}.
    \label{eq:sig}
\end{equation}

\subsection{Triviality Corrected Endogenous Reward}
Directly using the log-likelihood ratio $\phi(y_i\mid s_i)$ as a reward can be numerically unstable when $\pi_b(y_i\mid s_i)$ is extremely small. We therefore retain the endogenous reward $\log \pi_s(y_i\mid s_i)$ and add a gated triviality correction. The TCER reward is defined as
\begin{equation}
    r_t(y_i\mid s_i) = \log \pi_s(y_i\mid s_i) + k \cdot w(y_i\mid s_i) \cdot \phi(y_i\mid s_i),
    \label{eq:rt}
\end{equation}
with the triviality corrected weighting
\begin{equation}
    w(y_i\mid s_i) = \left(1-\pi_s(y_i\mid s_i)\right)^{\lambda},
    \label{eq:gate}
\end{equation}
where $k>0$ scales the correction and $\lambda>0$ controls the sharpness of suppression. For numerical stability, we apply $\varepsilon$-smoothing with $\varepsilon>0$:
\begin{equation}
    \phi(y_i\mid s_i) \leftarrow \log\frac{\pi_s(y_i\mid s_i)+\varepsilon}{\pi_b(y_i\mid s_i)+\varepsilon}.
    \label{eq:eps}
\end{equation}
Giving a context $s_i$ and a candidate token $y_i$, and denote $p=\pi_s(y_i\mid s_i)\in[0,1]$ and $q=\pi_b(y_i\mid s_i)\in[0,1]$. The gated correction term in Eq.~\eqref{eq:rt} can be written as
\begin{equation}
    C(p) = (1-p)^{\lambda}\cdot \log\frac{p+\varepsilon}{q+\varepsilon}.
\end{equation}
Under Eq.~\eqref{eq:eps}, the smoothed log-ratio is uniformly bounded: since $p+\varepsilon\in[\varepsilon,1+\varepsilon]$ and $q+\varepsilon\in[\varepsilon,1+\varepsilon]$, we have
\begin{equation}
    \left|\log\frac{p+\varepsilon}{q+\varepsilon}\right| \le \log\frac{1+\varepsilon}{\varepsilon}.
    \label{eq:log_ratio_bound}
\end{equation}
Therefore, the correction term admits the quantitative bound
\begin{equation}
    |C(p)| \le (1-p)^{\lambda}\,\log\frac{1+\varepsilon}{\varepsilon}.
    \label{eq:cp_bound}
\end{equation}
In particular, as $p\to 1$, $(1-p)^{\lambda}\to 0$ while the factor in Eq.~\eqref{eq:log_ratio_bound} remains finite, yielding
\begin{equation}
    \lim_{p\to 1} C(p)=0.
\end{equation}
Conversely, if $p \le 1-\delta$ for some $\delta\in(0,1]$, then $(1-p)^\lambda \ge \delta^\lambda$, and the correction term satisfies
\begin{equation}
    |C(p)| \ge \delta^\lambda \left|\log\frac{p+\varepsilon}{q+\varepsilon}\right|.
    \label{eq:cp_lower}
\end{equation}
Eq.~\eqref{eq:cp_lower} implies that when $\pi_s(y_i|s_i)$ is moderate, the gate function preserves the information gain signal at least a $\delta^\lambda$ fraction of the log-ratio magnitude, allowing the correction to reward domain-specific tokens while suppressing high-probability trivial.
% Eq.~\eqref{eq:cp_lower} implies that away from the saturated regime, the gate preserves at least a $\delta^\lambda$ fraction of the log-ratio magnitude, thus when the log-ratio indicate a in-domain preference shift, the correction can remain appreciable on non-saturated tokens.

\begin{table*}[ht]
\centering
\renewcommand{\arraystretch}{1.2}
\setlength{\tabcolsep}{2mm}
\begin{tabular}{lccccccccc}
\hline
\textbf{Model} & \textbf{LB} & \textbf{HB-A} & \textbf{HB-B} & \textbf{WB-A} & \textbf{WB-B} & \textbf{WB-C} & \textbf{WB-D} & \textbf{WB-E} & \textbf{WB-F} \\
\hline 
\multicolumn{10}{l}{\textit{General Closed-source Models}} \\
GPT-4o & 83.1 & 83.7 & 87.6 & 74.4 & 73.4 & 74.3 & 77.9 & 75.8 & 78.0 \\
Claude 3.5 & 89.3 & 82.9 & 88.3 & 59.5 & 57.6 & 56.3 & 59.3 & 62.0 & 67.7 \\
Claude 3.7 & 97.8 & 83.9 & 93.2 & 78.2 & 77.9 & 76.5 & 79.3 & 79.2 & 80.8 \\
\hline
\multicolumn{10}{l}{\textit{General Open-source Models}} \\
Llama3.1-8B-Instruct & 60.3 & 45.5 & 41.3 & 47.5 & 45.6 & 43.7 & 42.3 & 48.6 & 50.8 \\
Qwen2.5-32B-Instruct & 78.8 & 77.0 & 78.4 & 52.5 & 49.8 & 51.0 & 49.6 & 53.9 & 54.2 \\
Qwen3-8B & 85.2 & 81.4 & 85.3 & 68.7 & 68.9 & 67.0 & 67.2 & 71.2 & 71.3 \\
Qwen3-32B & 93.2 & 84.0 & 86.8 & 79.8 & 78.2 & 80.8 & 78.9 & 82.9 & 81.9 \\
\hline
\multicolumn{10}{l}{\textit{Finetuned Open-source Models}} \\
Qwen2.5-7B-Instruct & 74.7 & 72.0 & 76.3 & 58.9 & 56.8 & 56.5 & 54.0 & 59.9 & 60.0 \\
+ SFT & 83.1 & 73.7 & 85.8 & 69.7 & 69.8 & 67.9 & 63.1 & 71.3 & 66.3 \\
+ SFT + RL-EndoR & 84.5 & 79.0 & 84.7 & 72.6 & 70.9 & 70.1 & 64.6 & 73.0 & 68.7 \\
+ SFT + \textbf{RL-TCER} & \textbf{86.3} & \textbf{81.1} & \textbf{86.4} & \textbf{72.7} & \textbf{73.1} & \textbf{70.8} & \textbf{67.8} & \textbf{73.4} & \textbf{71.1} \\
\hline
\end{tabular}
\caption{Performance comparison on writing benchmarks. We evaluate models across LongBench-Write, HelloBench subsets, and WritingBench domains. Starting from Qwen2.5-7B-Instruct, each training stage (SFT, RL-EndoR, RL-TCER) demonstrates progressive improvements.}
\label{tab:main_writing}
\vspace{-4mm}
\end{table*}

\subsection{Reference Augmented GRPO}
We optimize the policy $\pi_\theta$ against $r_t$ using GRPO. For a given input prompt $q$, we sample a group of $G$ outputs $\{o_1, \dots, o_G\}$ from the current policy $\pi_{\theta_{\text{actor}}}$. To stabilize the advantage estimation and guide exploration, we augment this group with the ground-truth reference $o_{\text{gt}}$ from the dataset, forming the augmented set $\mathcal{O} = \{o_1, \dots, o_G, o_{\text{gt}}\}$. Although $r_t$ is defined at the token level, GRPO uses an output-level scalar reward. We define the sequence-level reward as the length-normalized average of token rewards:
\begin{equation}
    R_i = \frac{1}{|o_i|}\sum_{t=1}^{|o_i|} r_t(o_{i,t}).
\end{equation}
The advantage $\hat{A}_i$ is computed by normalizing against the augmented group statistics:
\begin{equation}
    \hat{A}_i = \frac{R_i - \frac{1}{|\mathcal{O}|}\sum_{j \in \mathcal{O}} R_j}{\sigma_{\mathcal{O}}}.
\end{equation}
The inclusion of $o_{\text{gt}}$ elevates the baseline mean, effectively penalizing trivial generations that fail to match the reference under the reward. This sequence-level advantage is broadcast to each token position, and maximize the GRPO objective:
\begin{equation}
\small
\begin{aligned}
    \mathcal{J}_{\text{GRPO}}(\theta) = \mathbb{E}\Big[ \min\big(&\rho_{i,t}(\theta)\hat{A}_i,\ 
    \text{clip}(\rho_{i,t}(\theta),1-\epsilon, \\
    & 1+\epsilon)\hat{A}_i\big) -\beta\, D_{\mathrm{KL}}(\pi_\theta\Vert\pi_{\mathrm{ref}})
    \Big].
\end{aligned}
\label{eq:grpo_compact}
\end{equation}
where $\rho_{i}(\theta)$ is the policy ratio, $\epsilon$ is the clipping parameter, and $\beta$ controls the KL divergence penalty to a reference policy $\pi_{\text{ref}}$.

~\section{Experiments}
\subsection{Experimental Setup}
\paragraph{Datasets.}
Our experimental setup utilizes distinct datasets tailored for writing and mathematical reasoning tasks to ensure specialized training and robust evaluation. For writing tasks, we employ two high-quality datasets:
\begin{itemize}
    \item \textbf{DeepWriting} \cite{wang2025reverse}: A collection of 20,000 high-quality writing samples designed to elicit complex and nuanced text generation.
    \item \textbf{LongWriter} \cite{bai2024longwriter}: A dataset comprising 6,000 samples focused on long-form, coherent content creation.
\end{itemize}
For both datasets, we implement a 50/50 split. The first half is used for the SFT phase to build an in-domain policy ($\pi_{\text{s}}$), while the second, disjoint half is used as the source of prompts for the RL phase. This separation prevents data leakage between the SFT and RL stages and ensures that the RL algorithm optimizes on unseen prompts. For mathematical reasoning tasks, we follow the data strategy established by LUFFY \cite{yan2025learning}. We use a 45,000-sample subset of the OpenR1-Math-220k dataset for the SFT phase. For the RL phase, we train our model on the Math-LightEval dataset.

\begin{table*}[ht]
\centering
\renewcommand{\arraystretch}{1.2}
\setlength{\tabcolsep}{2.5mm}
\begin{tabular}{lccccccccc}
\hline
\textbf{Model} & \textbf{LB} & \textbf{HB-A} & \textbf{HB-B} & \textbf{WB-A} & \textbf{WB-B} & \textbf{WB-C} & \textbf{WB-D} & \textbf{WB-E} & \textbf{WB-F} \\
\hline
\multicolumn{10}{l}{\textit{DeepWriting Dataset}} \\
DeepWriter-8B & 85.6 & 77.8 & 85.7 & 72.2 & 71.8 & 69.8 & 70.6 & 73.7 & 72.3 \\
+ RL-EndoR & 86.1 & 79.6 & \textbf{86.9} & 76.0 & 73.6 & 73.9 & 69.6 & \textbf{77.3} & 74.5 \\
+ \textbf{RL-TCER} & \textbf{87.3} & \textbf{82.2} & \underline{86.2} & \textbf{76.5} & \textbf{75.3} & \textbf{74.8} & \textbf{74.2} & \underline{76.7} & \textbf{75.2} \\
\hline
\multicolumn{10}{l}{\textit{LongWriter Dataset}} \\
LongWriter-8B & 76.5 & 68.8 & 82.9 & 57.9 & 53.9 & 49.0 & 52.0 & 52.9 & 52.0 \\
+ RL-EndoR & 80.2 & 70.1 & 84.2 & 64.4 & 62.6 & 60.2 & 56.7 & 59.9 & 58.8 \\
+ \textbf{RL-TCER} & \textbf{82.6} & \textbf{72.6} & \textbf{84.5} & \textbf{65.8} & \textbf{62.9} & \textbf{61.1} & \textbf{57.1} & \textbf{60.6} & \textbf{59.8} \\
\hline
\end{tabular}
\caption{Effectiveness of TCER for existing finetuned models. Applying our RL methods to finetuned models DeepWriter-8B and LongWriter-8B.}
\label{tab:robustness}
\vspace{-4mm}
\end{table*}

\paragraph{Benchmarks.}
To provide a comprehensive and multi-aspect evaluation of our models, we employ a suite of specialized benchmarks targeting both writing proficiency and mathematical reasoning. For Writing Tasks, we probe three distinct dimensions of generative performance using three complementary benchmarks following the evaluation protocol established in DeepWriter:
\begin{itemize}
    \item \textbf{LongBench-Write (LB)} \cite{bai2024longwriter}: This benchmark is designed to measure a model’s ability to produce ultra-long-form text, allowing us to assess its foundational capacity for maintaining thematic consistency over extended outputs.
    \item \textbf{HelloBench (HB)} \cite{que2024hellobench}: HelloBench evaluates performance on a diverse set of in-the-wild tasks. Our analysis focuses on two key subsets: \textbf{HB-A} (Open-Ended QA), which tests the generation of detailed and nuanced answers, and \textbf{HB-B} (Heuristic Text Generation), which assesses creative reasoning and stylistic fidelity in long-form narrative continuation.
    \item \textbf{WritingBench (WB)} \cite{wu2025writingbench}: This benchmark measures controllability across six writing domains: \textbf{A} (Academic \& Engineering), \textbf{B} (Finance \& Business), \textbf{C} (Politics \& Law), \textbf{D} (Literature \& Arts), \textbf{E} (Education), and \textbf{F} (Advertising \& Marketing).
\end{itemize}
Given the subjective nature of open-ended generation, we adopt the established protocol of using close-source LLMs as judges. Specifically, Claude-Sonent37 was used to score outputs for LongBench and WritingBench, while GPT-4o was employed for HelloBench. The reported scores represent win rates or the original benchmark scores without rescaling. All results are averaged over three independent runs.

For Mathematical Reasoning, we evaluate our models on six widely recognized benchmarks to test the cross-domain generalization of our method: AIME 2024 \cite{li2024numinamath}, AIME 2025 \cite{li2024numinamath}, AMC \cite{li2024numinamath}, MATH-500 \cite{hendrycks2021measuring}, Minerva \cite{lewkowycz2022solving} and OlympiadBench \cite{he2024olympiadbench}. For AIME 2024, AIME 2025, and AMC, which have smaller test sets, we report the avg@32 metric. For the other three larger benchmarks, we report standard pass@1 accuracy.

% \paragraph{Implementation Details.}
% Across all tasks, we establish a consistent experimental framework. For Writing Tasks, we conducted experiments on three powerful open-source models: Qwen2.5-7B-Instruct, Qwen3-8B, and Llama3.1-8B-Instruct. For Mathematical Reasoning, we conducted experiments on Qwen2.5-Math-7B and Llama3.1-8B-Instruct. The training parameters were tailored to each task domain:
% \begin{itemize}
%     \item For Writing Tasks, SFT was conducted on DeepWriter and LongWriter for 1 epoch with a learning rate of $2 \times 10^{-5}$ and a batch size of 8. The subsequent RL phase was run for 1,000 steps.
%     \item For Math Reasoning Tasks, we performed SFT for 3 epochs and RL for 500 steps. The learning rate was set to $1 \times 10^{-5}$ with a batch size of 64.
% \end{itemize}
% All training was performed on a cluster of 16x NVIDIA H20 GPUs. To ensure efficient reward computation during the RL phase, the actor policy ran on the main training GPUs, while the SFT policy ($\pi_{\text{SFT}}$) and the base policy ($\pi_{\text{Base}}$) were deployed on a separate set of 2x NVIDIA H20 GPUs. 

\begin{table*}[ht]
\centering
\renewcommand{\arraystretch}{1.2}
\setlength{\tabcolsep}{2mm}
\begin{tabular}{lccccccc}
\hline
\textbf{Method} & \textbf{AIME 24} & \textbf{AIME 25} & \textbf{AMC} & \textbf{MATH-500} & \textbf{Minerva} & \textbf{Olympiad} & \textbf{Avg} \\
\hline
Qwen2.5-7B-Math & 11.2 & 4.4 & 30.8 & 46.0 & 7.7 & 14.4 & 19.1 \\
+ SFT & 22.2 & 22.3 & 52.8 & 82.6 & 40.8 & 43.7 & 44.1 \\
+ SFT + RL-EndoR & 30.5 & 25.4 & 60.9 & \textbf{87.0} & 42.6 & 46.8 & 48.9 \\
+ SFT + \textbf{RL-TCER} & \textbf{32.4} & \textbf{26.1} & \textbf{62.1} & \underline{86.4} & \textbf{44.5} & \textbf{49.6} & \textbf{50.2} \\
\hline
Llama3.1-8B-Instruct & 4.6 & 0.2 & 21.2 & 46.4 & 20.9 & 12.4 & 17.6 \\
+ SFT & 9.4 & 9.5 & 39.0 & 68.4 & 27.2 & 32.4 & 30.9 \\
+ SFT + RL-EndoR & 9.7 & 13.2 & 39.2 & 70.0 & 25.7 & 33.3 & 31.9 \\
+ SFT + \textbf{RL-TCER} & \textbf{10.2} & \textbf{15.0} & \textbf{40.5} & \textbf{71.8} & \textbf{28.6} & \textbf{35.3} & \textbf{33.6} \\
\hline
\end{tabular}
\caption{Performance comparison on mathematical reasoning benchmarks. TCER outperforms both the SFT baseline and EndoR across different base models.}
\label{tab:math}
\vspace{-3mm}
\end{table*}

\subsection{Main Results}
Table~\ref{tab:main_writing} details the results of our full training pipeline, starting from the base Qwen2.5-7B-Instruct model and applying SFT and RL sequentially on the DeepWriting dataset. Fine-tuning on the DeepWriting dataset boosts the model's performance across all benchmarks compared to the original Qwen2.5-7B-Instruct model. For instance, the LongBench (LB) score jumps from 74.7 to 83.1, and the average WritingBench (WB) performance is significantly elevated, demonstrating the effectiveness of domain specialization. Applying RL with the endogenous reward provides further improvements over the SFT model, confirming that endogenous rewards are a viable signal. However, its gains are inconsistent and marginal in several creativity intensive dimensions. Our proposed TCER method consistently outperforms all baselines, achieving the highest scores across all reported benchmarks for our finetuned models. 

To assess the applicability of TCER as a post-training enhancement technique, we applied our RL method to finetuned open-source writing models: DeepWriter-8B and LongWriter-8B. These models represent a high standard of performance. The results, shown in Table~\ref{tab:robustness}, demonstrate that TCER can further elevate their capabilities. For the DeepWriter-8B model, TCER boosts the score from 70.6 to \textbf{74.2} on WritingBench-D (Literature \& Arts), and on HelloBench-A (Ability), it increases the score from 77.8 to \textbf{82.2}. Similarly, for the LongWriter-8B model, TCER provides a targeted boost to its primary strength. The LongBench (LB) score improves from 76.5 to \textbf{82.6} after applying TCER, a substantial (+6.1) point increase that surpasses the gain from RL-EndoR.

\begin{figure}[ht]
    \centering
    \includegraphics[width=0.48\textwidth]{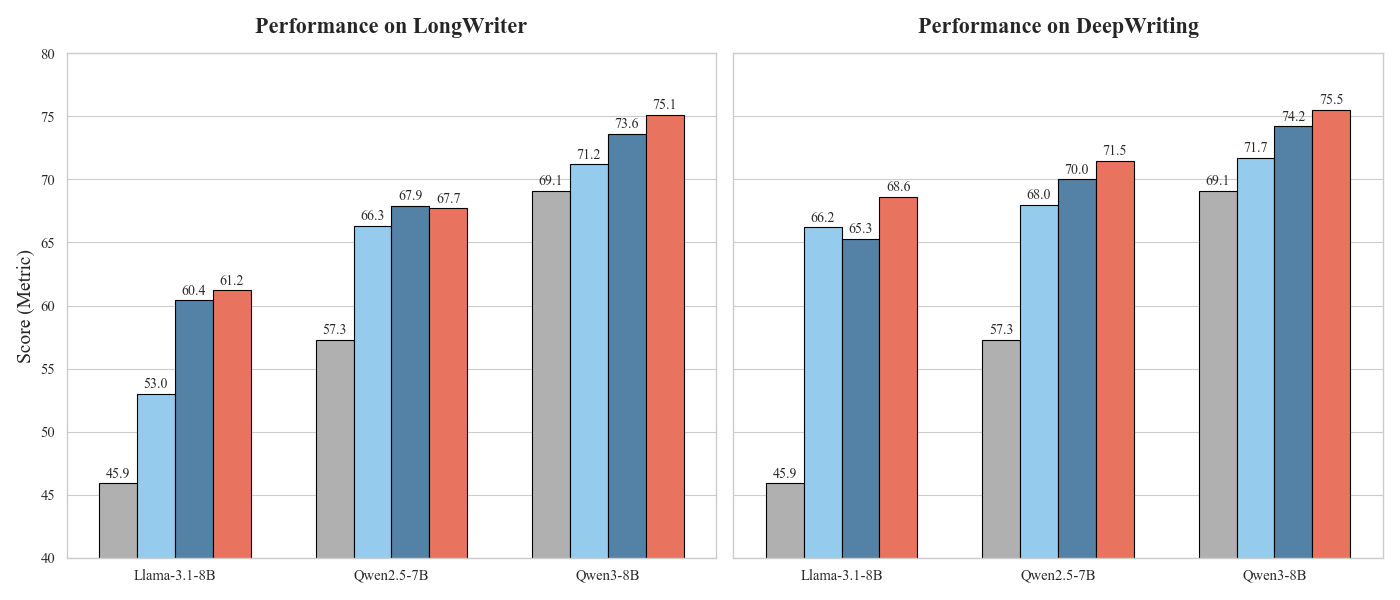}
    \caption{Performance generalization across different models and datasets. For each model group, the bars represent the performance of the base model (grey), after SFT (light blue), after applying RL-EndoR (dark blue), and after applying our RL-TCER (orange).}
    \label{fig:model_generalization}
    \vspace{-4mm}
\end{figure}

\subsection{Generalization to Different Models}
To verify that the benefits of TCER are not specific to a single model architecture, we conducted a comparative analysis across three distinct open-source models: Llama-3.1-8B-Instruct, Qwen2.5-7B-Instruct, and Qwen3-8B. Each model was subjected to the same SFT and RL pipeline on both the LongWriter and DeepWriting datasets. The results, visualized in Figure~\ref{fig:model_generalization}, demonstrate the robust, model-agnostic nature of our approach. A clear and consistent trend emerges across all six configurations: our method consistently achieves the highest performance, surpassing both the SFT baseline and the EndoR. These results indicate that TCER is a applicable enhancement, offering performance gains regardless of the underlying model architecture.

\subsection{Generalization to Math Reasoning}
To investigate the cross-domain applicability of our method, we extended our evaluation to the verifiable and logical domain of mathematical reasoning. We applied the same SFT and RL training pipeline to two different base models, Qwen2.5-7B and Llama3.1-8B-Instruct. The results, presented in Table~\ref{tab:math}, demonstrate that TCER's benefits are not confined to open-end generation but generalize effectively to mathematical reasoning. For both the Qwen and Llama models, our RL-TCER method consistently achieves the average scores, outperforming both the SFT baseline and the RL-EndoR approach. This suggests that by penalizing trivial, high-probability steps and rewarding more diversity ones, TCER encourages the model to explore potentially correct solution paths.

\begin{table}[t]
\centering
\renewcommand{\arraystretch}{1.2}
\setlength{\tabcolsep}{2.5mm}
\begin{tabular}{l|cc}
\hline
\textbf{Metric} & \textbf{EndoR} & \textbf{TCER} \\
\hline
High-quality avg. & -1.11 & \textbf{-0.80}\\
Regular avg. & -1.08 & -0.93 \\
High-quality Recall \% & 20.7\% & \textbf{35.8\%} \\
\hline
\end{tabular}
\caption{Sentence-level analysis of reward signals. Sentences identified as brilliant by Gemini-2.5-Pro, GPT-4o, Claude-Opus4 are evaluated. Analysis conducted on 10,000 outputs with $\geq$5 high-quality sentences per text.}
\label{tab:micro_validation}
\vspace{-4mm}
\end{table}

\begin{table*}[ht]
\centering
\renewcommand{\arraystretch}{1.2}
\setlength{\tabcolsep}{1.8mm}
\begin{tabular}{lccccccccc}
\hline
\textbf{Config} & \textbf{LB} & \textbf{HB-A} & \textbf{HB-B} & \textbf{WB-A} & \textbf{WB-B} & \textbf{WB-C} & \textbf{WB-D} & \textbf{WB-E} & \textbf{WB-F} \\
\hline
\textbf{FULL TCER} & \textbf{86.3} & \textbf{81.1} & \textbf{86.4} & \textbf{72.7} & \textbf{73.1} & \textbf{70.8} & \textbf{67.8} & \textbf{73.4} & \textbf{71.1} \\
\hline
w/o $w(y_i|s_i) \cdot \phi(y_i|s_i)$ & 84.5 & 79.0 & 84.7 & 72.6 & 70.9 & 70.1 & 64.6 & 73.0 & 68.7 \\
w/o $\log \pi_s(y_i|s_i)$ & 76.2 & 70.8 & 72.0 & 68.5 & 68.2 & 65.6 & 60.1 & 69.4 & 64.1 \\
w/o $w(y_i|s_i)$ & 81.5 & 76.5 & 79.1 & 70.6 & 70.5 & 71.2 & 62.7 & 73.1 & 65.3 \\
w/o Ref-Augmentation & 85.7 & 78.2 & 85.6 & 72.9 & 72.3 & 70.6 & 66.6 & 73.3 & 71.2 \\
\hline
\end{tabular}
\caption{Ablation studies on Qwen2.5-7B-Instruct with the DeepWriting dataset. Each component contributes to the overall performance.}
\label{tab:ablation}
\vspace{-4mm}
\end{table*}

\subsection{Validation of Quality Enhancement}
To validate that TCER identifies and rewards high quality content, we conducted a sentence-level analysis comparing the reward signals of TCER and EndoR. First, we prompted the SFT model to generate a diverse set of outputs across writing tasks. Second, we employed a panel of three LLMs including Gemini 2.5 Pro, Claude-Opus4, and GPT-4o as judges. Through aggregation by designed prompt (with Gemini 2.5 Pro performing the aggregation), these judges identified sentences exhibiting exceptional writing quality. Third, for each sentence in the corpus, we computed both the EndoR score and the TCER score (scores are aggregated log-probs; higher is closer to $0$). While our reward mechanism operates at the token level, we aggregate token rewards within each sentence for interpretability, since training uses sequence-level aggregation of token rewards and sentence-level trends reveal where rewards concentrate in the generated text.

We evaluated three key metrics: (1) the average reward score for sentences identified as high-quality, (2) the average reward score for regular sentences, and (3) the recall rate of high-quality sentences. For the recall metric, we select the top-k sentences based on reward scores within each generated text, where k equals the number of high-quality sentences identified by the judge panel for that text, then calculate the recall rate for each text and report the average across all texts in Table\ref{tab:micro_validation}. We visualized the sentence-level reward trajectories for eight SFT outputs in Figure\ref{fig:Validation_of_Quality_Enhancement}. The visualizations reveal a consistent pattern: TCER and EndoR show notable divergence specifically at high-quality sentences identified by judges, with TCER assigning higher rewards at these points. In contrast, for regular sentences, the two methods produce comparable scores, indicating that TCER selectively amplifies rewards for exceptional content while maintaining similar baseline evaluations.

\begin{figure}[ht]
    \centering
    \includegraphics[width=0.49\textwidth]{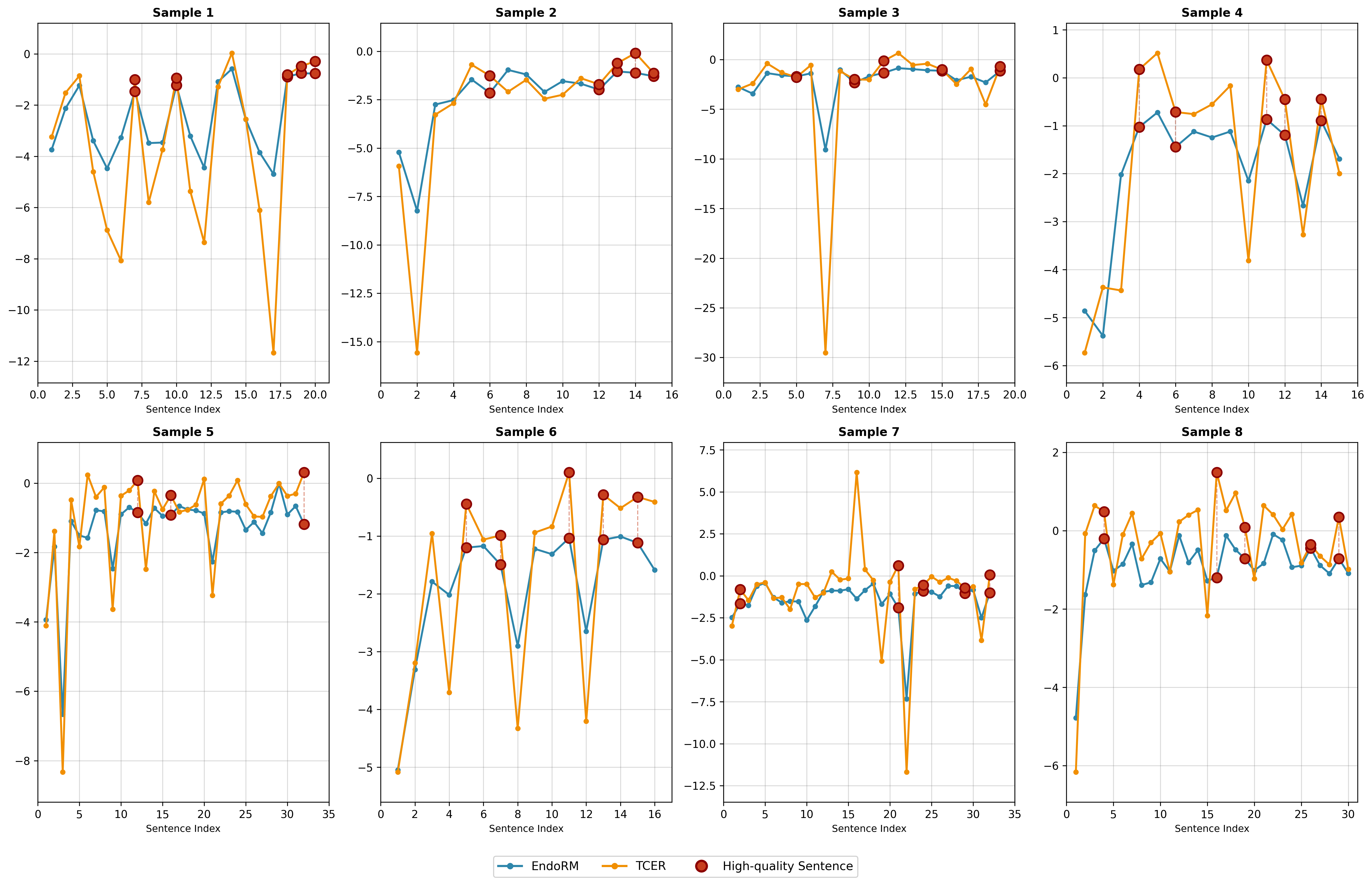}
    \caption{Displays these trajectories as line plots, with EndoR scores shown in blue and TCER scores in orange. High-quality sentences identified by our judge panel are marked with bold red points.}
    \label{fig:Validation_of_Quality_Enhancement}
    \vspace{-4mm}
\end{figure}

\subsection{Ablation Study}
To dissect the TCER formulation and isolate the contribution of its key components, we conducted a comprehensive ablation study. The experiments were performed on the Qwen2.5-7B-Instruct model with the DeepWriting dataset configuration. The results, detailed in Table~\ref{tab:ablation}.
\begin{itemize}
    \item \textbf{Removing the triviality correction:} Our method reverts to the EndoR baseline. This results in a performance drop across most benchmarks. This validates that the correction term is essential for mitigating triviality bias.
    \item \textbf{Removing the endogenous reward:} Relying only on the correction signal leads to a catastrophic performance collapse. This shows that $\log \pi_s(y_i|s_i)$ is crucial for maintaining overall generation quality and coherence.
    \item \textbf{Removing the gating coefficient:} Removing $w(y_i|s_i)$ causes a considerable performance drop, indicating the importance of selectively applying the correction to novel, low-probability tokens.
    \item \textbf{Removing the reference augmented:} Removing the ground-truth reference from GRPO results in lower performance, supporting the role of reference samples in elevating the optimization baseline.
\end{itemize}
\vspace{-4mm}

~\section{Conclusion}

We presented the study of transferring unsupervised RL methods from mathematical reasoning to high-quality writing. While endogenous rewards improve performance, they introduce Triviality Bias suppressing diversity and quality. Our proposed TCER addresses this through information theoretic rewards that measure divergence from base models, modulated by triviaity corrected weighting. Experiments demonstrate TCER's effectiveness across both writing and mathematical domains, achieving quality improvements. By eliminating dependence on external reward models, TCER offers a practical path toward scalable, high-quality text generation without costly human annotation evalution or API access.

\section*{Limitations}
Despite the demonstrated effectiveness of TCER, several limitations warrant discussion. First, as a self-supervised reinforcement learning method, while it successfully optimizes the model's existing capabilities and improves output quality, it has an inherent ceiling for endogenous reward approach. Second, our evaluation relies on automated benchmarks that employ LLMs as judges. While we acknowledge potential biases in using model-based evaluation, these benchmarks represent the current standard for scalable assessment of generation quality. Third, our analysis of reward attribution at the token level remains limited. Although TCER computes rewards based on token-level probabilities, the complex interdependencies in natural language limit the interpretability of why certain generation patterns receive higher rewards, making it difficult to provide fine-grained insights into the model's decision-making process. Future research should investigate the integration of endogenous and external rewards, alongside developing improved evaluation methodologies and interpretability techniques for deeper insights into model behavior.

\bibliography{custom}

\clearpage
\appendix
% \onecolumn
~\appendix
\vspace{-15pt}
\section{Additional Derivations and Proofs}
\label{app:derivations}

\subsection{Linearity of the one step objective and deterministic optimizer}
\label{app:linear_vertex}
We justify the claim in the main text that the one-step objective in Eq.~\eqref{eq:one_step_obj}
is linear in the policy $\pi(\cdot\mid s)$ for a fixed context $s$, and that an optimum is attained by a deterministic distribution. Fix a context $s$ and let $\mathcal{V}$ be the vocabulary. Let
\begin{equation}
    f(v) := \log \pi_s(v\mid s),  v\in\mathcal{V}.
\end{equation}
The one-step objective can be written as
\begin{equation}
    J(\pi;\,s) := \mathbb{E}_{y\sim \pi(\cdot\mid s)}[f(y)] = \sum_{v\in\mathcal{V}} \pi(v\mid s)\, f(v),
    \label{eq:app_J_linear}
\end{equation}
where $\pi(\cdot\mid s)$ ranges over the probability simplex
\begin{equation}
    \Delta(\mathcal{V}) := \Big\{\pi:\mathcal{V}\to[0,1]\ \big|\ \sum_{v\in\mathcal{V}}\pi(v)=1\Big\}.
\end{equation}

\begin{lemma}[Linearity in $\pi$]
\label{lem:linearity}
For any two distributions $\pi_1,\pi_2\in\Delta(\mathcal{V})$ and any $\lambda\in[0,1]$,
\begin{equation}
\begin{aligned}
    J(\lambda \pi_1 + (1-\lambda)\pi_2;\,s) &= \lambda J(\pi_1;\,s) \\
    &\quad + (1-\lambda)J(\pi_2;\,s).
\end{aligned}
\end{equation}
\end{lemma}

\begin{proof}
By Eq.~\eqref{eq:app_J_linear},
\begin{align}
J(&\lambda \pi_1 + (1-\lambda)\pi_2;\,s) \nonumber\\
&= \sum_{v\in\mathcal{V}} \big(\lambda \pi_1(v) + (1-\lambda)\pi_2(v)\big) f(v) \nonumber\\
&= \lambda \sum_{v\in\mathcal{V}} \pi_1(v) f(v) \nonumber + (1-\lambda)\sum_{v\in\mathcal{V}} \pi_2(v) f(v) \nonumber\\
&= \lambda J(\pi_1;\,s) + (1-\lambda)J(\pi_2;\,s).
\end{align}
\end{proof}

\begin{lemma}[Deterministic optimizer]
\label{prop:det_opt}
The maximum of $J(\pi;\,s)$ over $\Delta(\mathcal{V})$ is achieved by a deterministic distribution. In particular,
\begin{equation}
    \max_{\pi\in\Delta(\mathcal{V})} J(\pi;\,s)
    =
    \max_{v\in\mathcal{V}} f(v),
\end{equation}
and any optimizer is of the form $\pi^*(\cdot\mid s)=\delta_{v^*}(\cdot)$ where $\delta_{v^*}$ denotes the Dirac distribution at $v^*\in\arg\max_{v\in\mathcal{V}} f(v)$.
\end{lemma}

\begin{proof}
Let $M:=\max_{v\in\mathcal{V}} f(v)$. For any $\pi\in\Delta(\mathcal{V})$,
\begin{equation}
    J(\pi;\,s)=\sum_{v\in\mathcal{V}}\pi(v)f(v)\le \sum_{v\in\mathcal{V}}\pi(v)M=M,
\end{equation}
since $\sum_v \pi(v)=1$. Equality holds iff $\pi$ places all mass on the set of maximizers of $f$, i.e., $\pi(v)=0$ for all $v$ such that $f(v)<M$. In particular, choosing any single maximizer $v^*$ and setting $\pi(v^*)=1$ achieves $J(\pi;\,s)=M$.
\end{proof}

The above results demonstrate that optimizing the endogenous reward without regularization inevitably leads to deterministic policies, providing theoretical support for the entropy collapse phenomenon observed in practice.

\subsection{Coverage of the information gain correction}
\label{app:coverage}

The main text uses the intuition that TCER does not only suppress saturated tokens but also affects \emph{how broadly} the information-gain signal can contribute within a context. We formalize this via a context-wise coverage quantity that measures the expected activity of the gating function.

\paragraph{Definition.}
Given a context $s$, define the gate $w(y\mid s)=(1-\pi_s(y\mid s))^\lambda\in[0,1]$ and the context-wise coverage
\begin{equation}
\begin{aligned}
    S(s) &:= \mathbb{E}_{y\sim\pi_s(\cdot\mid s)}[w(y\mid s)] \\
    &= \sum_{v\in\mathcal{V}} \pi_s(v\mid s)\big(1-\pi_s(v\mid s)\big)^\lambda.
\end{aligned}
\label{eq:app_S}
\end{equation}
The coverage $S(s)$ quantifies how much the correction term remains active on average under the distribution $\pi_s(\cdot\mid s)$.

\begin{lemma}[Range and deterministic case]
\label{lem:S_range}
For any context $s$, $S(s)\in[0,1]$. Moreover, if $\pi_s(\cdot\mid s)$ is deterministic (i.e., $\exists v^*$ s.t.\ $\pi_s(v^*\mid s)=1$), then $S(s)=0$.
\end{lemma}

\begin{proof}
Since $w(y\mid s)\in[0,1]$ for all $y$, its expectation under any distribution lies in $[0,1]$, hence $S(s)\in[0,1]$. If $\pi_s$ is deterministic, then with probability one the sampled token has $p=\pi_s(y\mid s)=1$, so $w=(1-1)^\lambda=0$ almost surely, yielding $S(s)=0$.
\end{proof}

\begin{lemma}[Coverage lower bound via non-saturated mass]
\label{lem:S_lower}
Fix $\delta\in(0,1]$ and define the non-saturated set
\begin{equation}
    A_\delta(s) := \{y\in\mathcal{V}:\ \pi_s(y\mid s)\le 1-\delta\}.
\end{equation}
Then
\begin{equation}
    S(s) \ge \delta^\lambda \cdot \pi_s\big(A_\delta(s)\mid s\big).
    \label{eq:app_S_lower}
\end{equation}
\end{lemma}

\begin{proof}
For any $y\in A_\delta(s)$, we have $1-\pi_s(y\mid s)\ge \delta$, hence $w(y\mid s)\ge \delta^\lambda$. Therefore,
\begin{align}
S(s) &= \mathbb{E}_{y\sim\pi_s}[w(y\mid s)] \nonumber\\
&\ge \mathbb{E}_{y\sim\pi_s}\big[w(y\mid s)\,\mathbf{1}\{y\in A_\delta(s)\}\big] \nonumber\\
&\ge \mathbb{E}_{y\sim\pi_s}\big[\delta^\lambda\,\mathbf{1}\{y\in A_\delta(s)\}\big] \nonumber\\
&= \delta^\lambda \pi_s(A_\delta(s)\mid s).
\end{align}
\end{proof}
Eq.~\eqref{eq:app_S_lower} shows that $S(s)$ is large whenever $\pi_s(\cdot\mid s)$ assigns substantial probability mass to non-saturated alternatives. In such contexts, the gate is less likely to shut off the information-gain term, so the $\phi$-based correction can apply to a broader portion of likely tokens under $\pi_s$. This provides a quantitative characterization of how TCER maintains active correction signals in high-entropy regimes.

\subsection{Closed-form derivative of the gated correction}
\label{app:correction_derivative}
For completeness, we provide the derivative of $C(p)$ with respect to $p$, which can be useful for analyzing smoothness:
\begin{equation}
    C(p)=(1-p)^{\lambda}\log\frac{p+\varepsilon}{q+\varepsilon}.
\end{equation}
A direct differentiation gives
\begin{equation}
\small
    \frac{d}{dp}C(p)
    =
    -\lambda(1-p)^{\lambda-1}\log\frac{p+\varepsilon}{q+\varepsilon}
    +
    \frac{(1-p)^{\lambda}}{p+\varepsilon}.
\end{equation}
Under $\varepsilon>0$ and $p\in[0,1]$, both $\log\frac{p+\varepsilon}{q+\varepsilon}$ and $(p+\varepsilon)^{-1}$ are bounded; thus the derivative remains bounded for $\lambda\ge 1$ and exhibits the expected suppression near $p\approx 1$ when $\lambda>1$.

While the derivative vanishes as $p \to 1$, suggesting diminishing correction for high-confidence tokens, the entropy-preserving effect of TCER operates through the sequence-averaged reward structure. The sequence reward is computed as:
\begin{equation}
    r_{\text{t}}(\tau) = \frac{1}{T}\sum_{t=1}^{T} \left[\log p_t + k \cdot C(p_t)\right]
\end{equation}
This averaged reward is then applied uniformly to all tokens in the sequence during gradient updates. 

To quantify the diversity preference, consider the expected reward difference between diverse and deterministic sequences. For a deterministic sequence where all $p_t \approx 1-\epsilon$ for small $\epsilon$:
\begin{equation}
\small
    r_{\text{det}} \approx \log(1-\epsilon) + k \cdot \epsilon^\lambda \cdot \log\frac{1-\epsilon+\varepsilon}{q+\varepsilon} \approx -\epsilon
\end{equation}
For a diverse sequence with confidence variance $\sigma^2 = \text{Var}_t[p_t]$, containing tokens with $p_t \in [1-2\sigma, 1]$:
\begin{equation}
\small
    r_{\text{diverse}} \geq \mathbb{E}[\log p_t] + k \cdot \mathbb{E}[(1-p_t)^\lambda \cdot \log\frac{p_t+\varepsilon}{q_t+\varepsilon}]
\end{equation}
The second term is substantial for medium-confidence tokens. Using the lower bound from Eq.~\eqref{eq:cp_lower}, tokens with $p_t \leq 1-\sigma$ contribute at least:
\begin{equation}
    k \cdot \sigma^\lambda \left|\log\frac{p_t+\varepsilon}{q_t+\varepsilon}\right|
\end{equation}
Therefore, the reward advantage of diverse sequences is:
\begin{equation}
\small
    \Delta r = r_{\text{diverse}} - r_{\text{det}} \geq k \cdot \sigma^{\lambda+1} \cdot \mathbb{E}\left[\left|\log\frac{p+\varepsilon}{q+\varepsilon}\right|\right] > 0
\end{equation}

This quantifies how the gating mechanism creates systematic preference for diversity: sequences with higher confidence variance receive higher average rewards, leading the policy to learn diverse outputs rather than collapse to deterministic patterns. The parameter $k$ controls the strength of this diversity preference, with larger $k$ providing stronger protection against entropy collapse.

\section{Experimental Details}

\subsection{Implementation Details}
Across all tasks, we adopt a unified training pipeline consisting of SFT followed by RL with Reference-Augmented GRPO. For writing, we experiment with Qwen2.5-7B-Instruct, Qwen3-8B, and Llama3.1-8B-Instruct; for mathematical reasoning, we use Qwen2.5-Math-7B and Llama3.1-8B-Instruct. On writing datasets (DeepWriter and LongWriter), we perform SFT for 1 epoch with learning rate $2\times 10^{-5}$ and batch size 8, and then run RL for 1{,}000 update steps. On math tasks, we run SFT for 3 epochs and RL for 500 update steps with learning rate $1\times 10^{-5}$ and batch size 64.

During RL, for each prompt we sample a group of $G=\text{8}$ rollouts and augment it with the dataset reference completion, using the length-normalized average token reward as the sequence reward. Following the theoretical requirement from EndoR that reward models should remain frozen during RL training to prevent distribution shift, we keep both $\pi_{\text{s}}$ and $\pi_{\text{b}}$ fixed throughout optimization. GRPO uses a KL penalty of coefficient $\beta=\text{0.001}$. TCER is computed using a correction scale $k=\text{3}$ and gating exponent $\lambda=\text{2}$, with $\varepsilon$-smoothing set to $\varepsilon=1\times 10^{-5}$ in the log-ratio for numerical stability. For generation during RL training, we use temperature $\tau=0.7$ and maximum length of 8000 tokens. All training data and evaluation benchmarks employ their original open-source prompts without modification. We will release our code upon acceptance to facilitate reproducibility. All experiments are run on a cluster of $16\times$ NVIDIA H20 GPUs; to reduce RL-time overhead, the actor runs on the main training GPUs while the frozen $\pi_{\text{s}}$ and $\pi_{\text{b}}$ are served on an additional $2\times$ H20 GPUs for reward computation.

\begin{figure*}[!t]
    \centering
    % ================= Row 1: DeepWriting =================
    \begin{subfigure}[t]{0.32\textwidth}
        \centering
        \includegraphics[width=\linewidth]{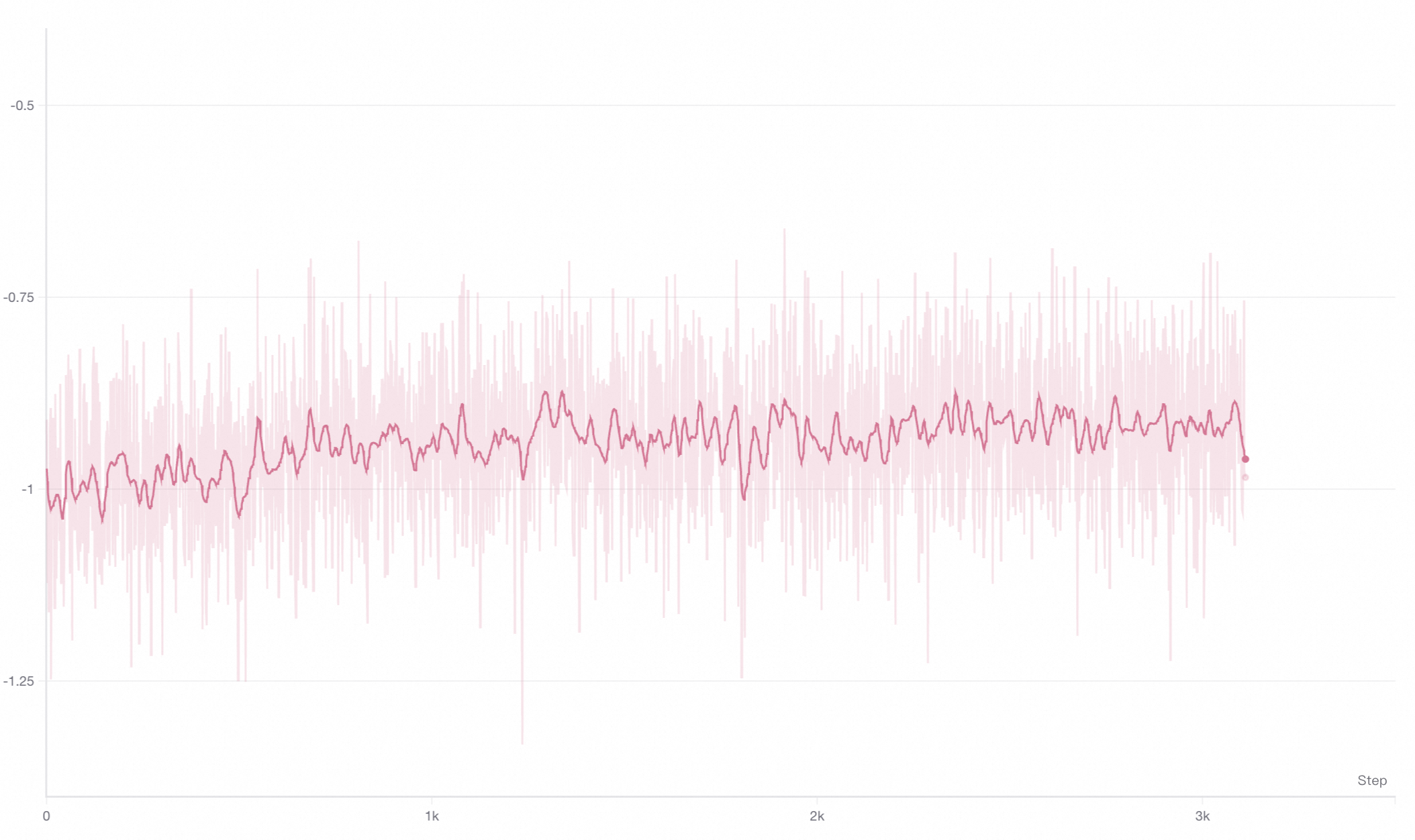}
        \caption{DeepWriting: EndoR reward.}
    \end{subfigure}\hfill
    \begin{subfigure}[t]{0.32\textwidth}
        \centering
        \includegraphics[width=\linewidth]{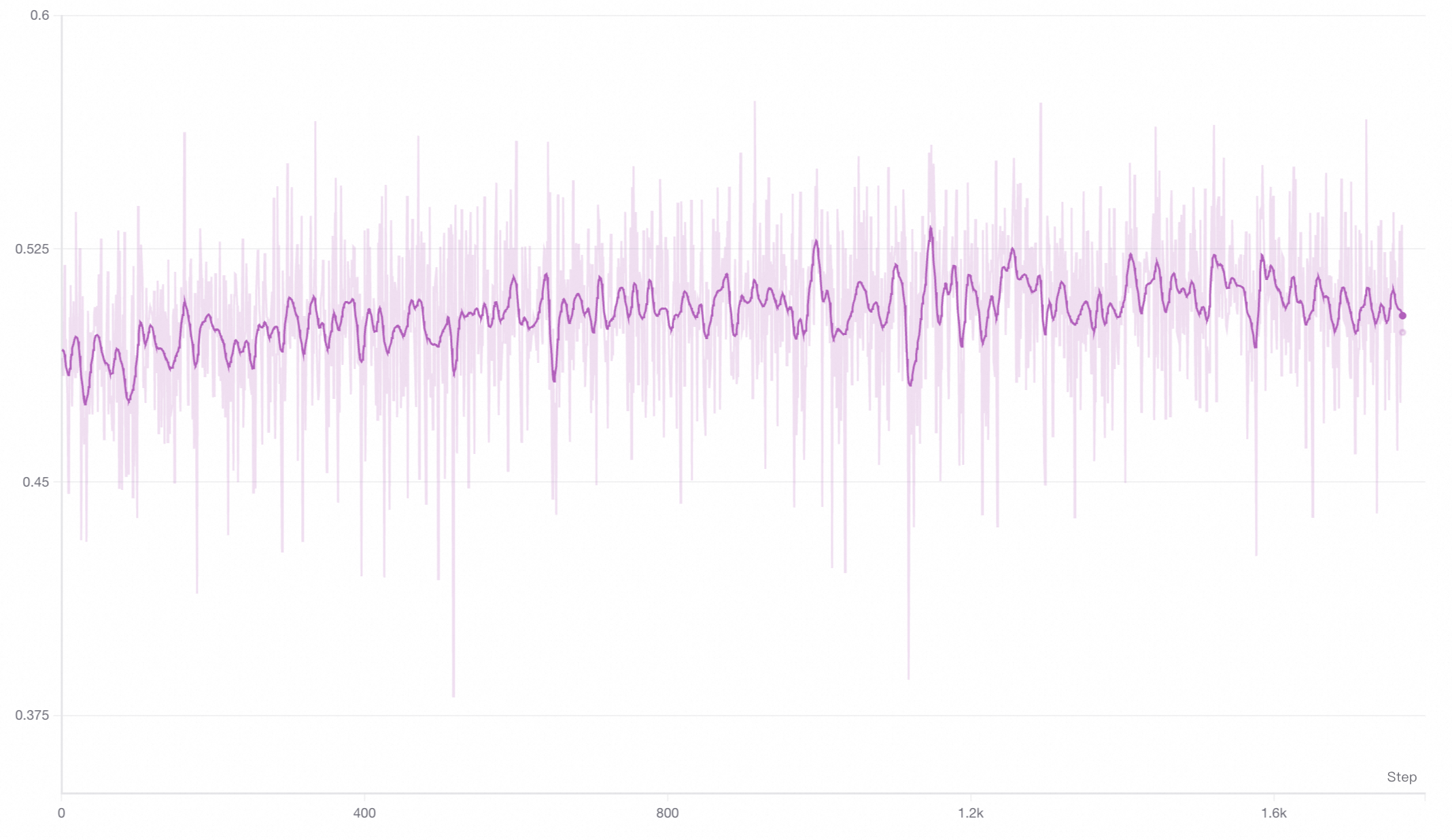}
        \caption{DeepWriting: TCER reward.}
    \end{subfigure}\hfill
    \begin{subfigure}[t]{0.32\textwidth}
        \centering
        \includegraphics[width=\linewidth]{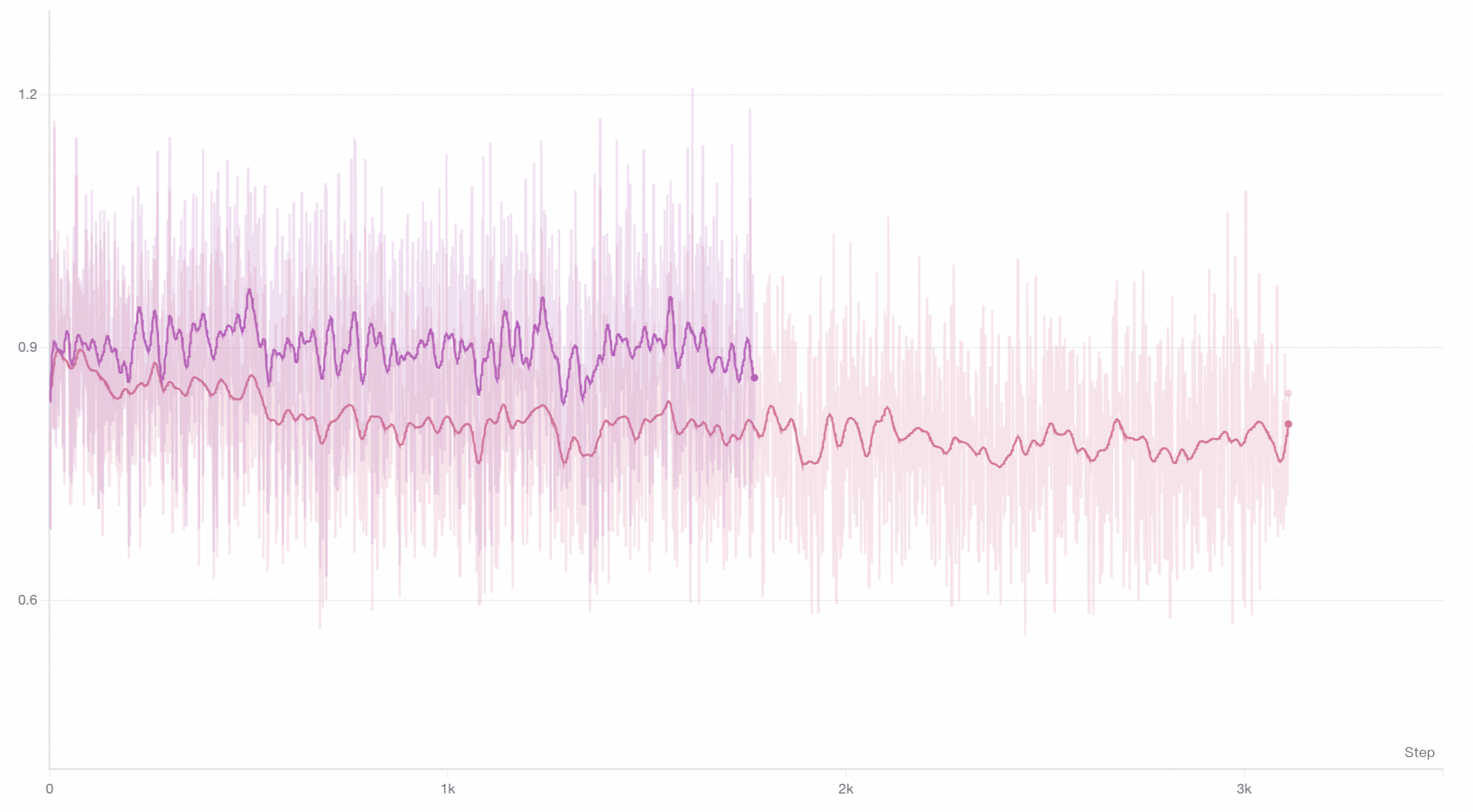}
        \caption{DeepWriting: entropy (EndoR vs TCER).}
    \end{subfigure}

    \vspace{0.5em}

    % ================= Row 2: LongWriter ==================
    \begin{subfigure}[t]{0.32\textwidth}
        \centering
        \includegraphics[width=\linewidth]{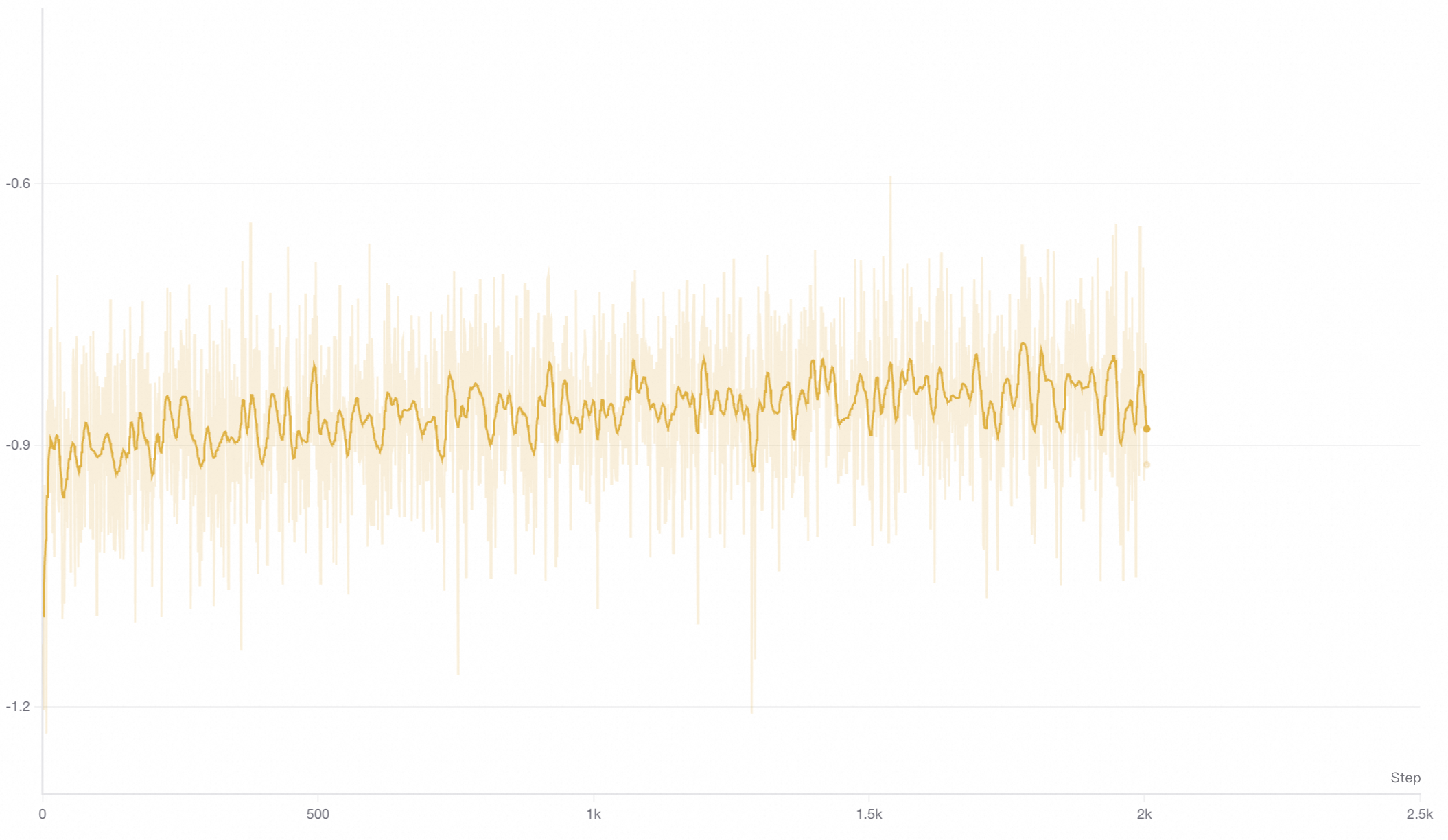}
        \caption{LongWriter: EndoR reward.}
    \end{subfigure}\hfill
    \begin{subfigure}[t]{0.32\textwidth}
        \centering
        \includegraphics[width=\linewidth]{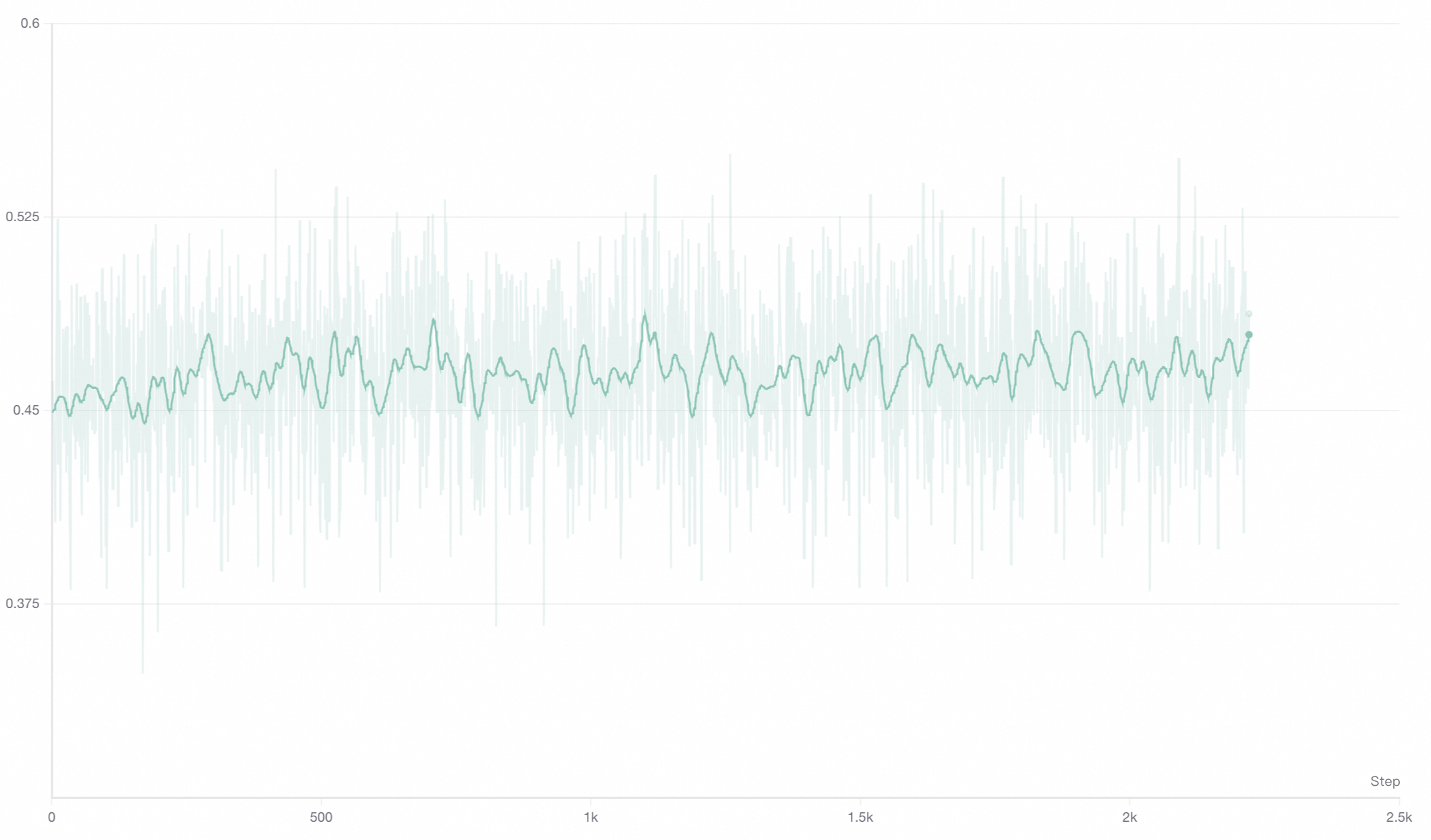}
        \caption{LongWriter: TCER reward.}
    \end{subfigure}\hfill
    \begin{subfigure}[t]{0.32\textwidth}
        \centering
        \includegraphics[width=\linewidth]{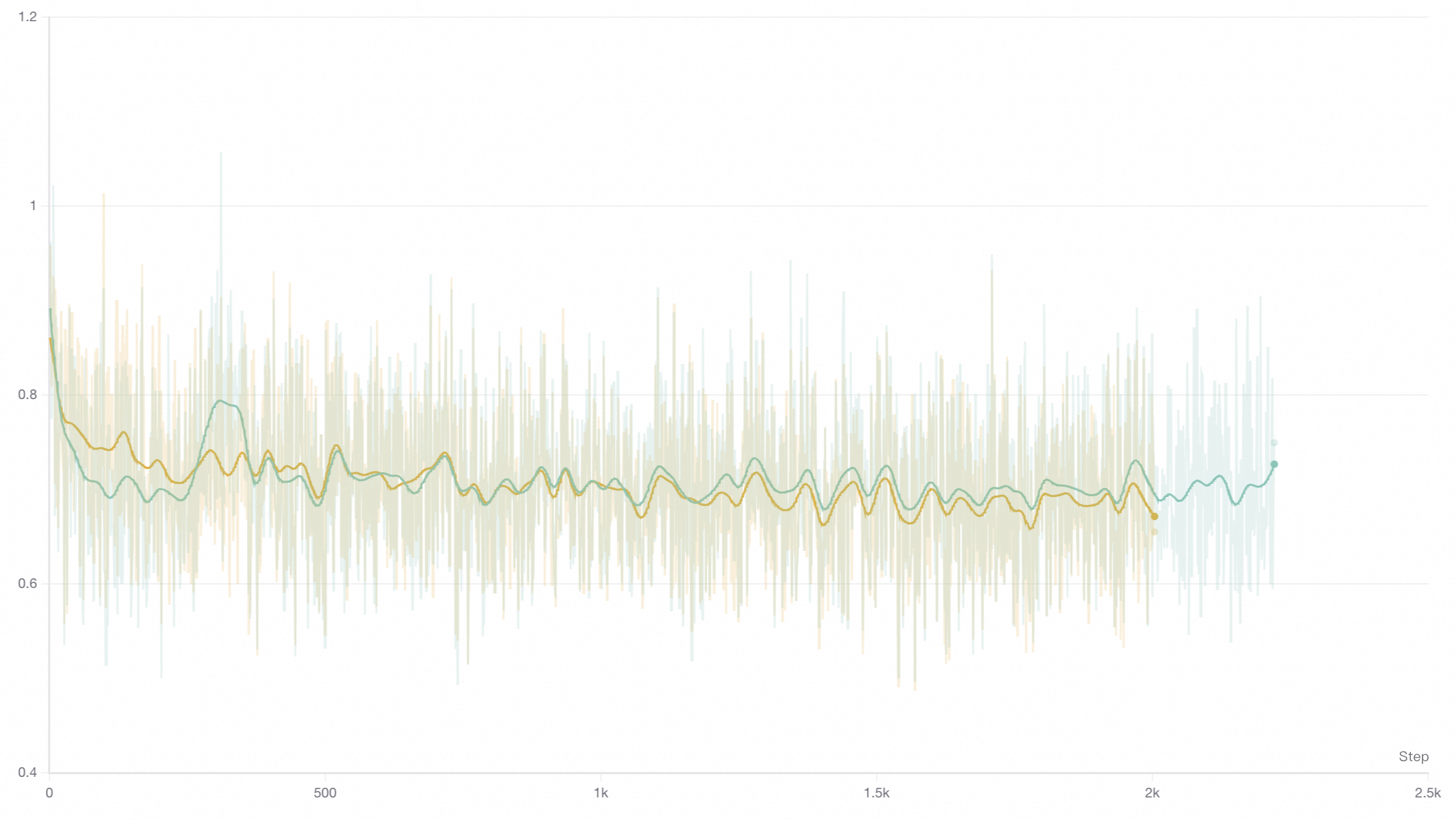}
        \caption{LongWriter: entropy (EndoR vs TCER).}
    \end{subfigure}

    \caption{RL training dynamics on writing tasks. For each dataset, we report the EndoR reward trajectory, the TCER reward trajectory , and a direct entropy comparison between EndoR and TCER .}
    \label{fig:rl_dynamics_writing_grid}
\end{figure*}

\begin{figure*}[!t]
    \centering
    % ================= Row 1: Qwen-Math ===================
    \begin{subfigure}[t]{0.32\textwidth}
        \centering
        \includegraphics[width=\linewidth]{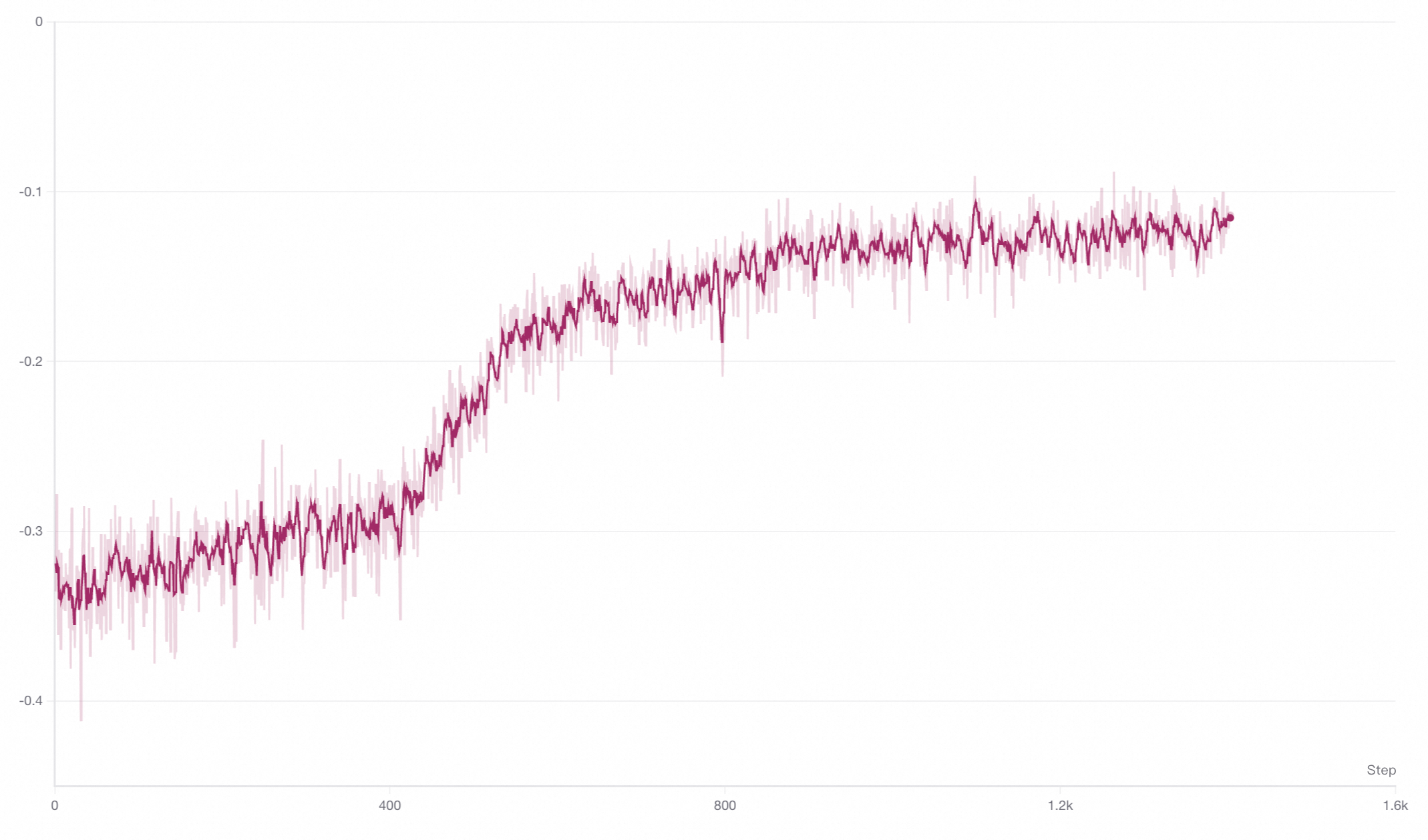}
        \caption{Qwen2.5-Math: EndoR reward.}
    \end{subfigure}\hfill
    \begin{subfigure}[t]{0.32\textwidth}
        \centering
        \includegraphics[width=\linewidth]{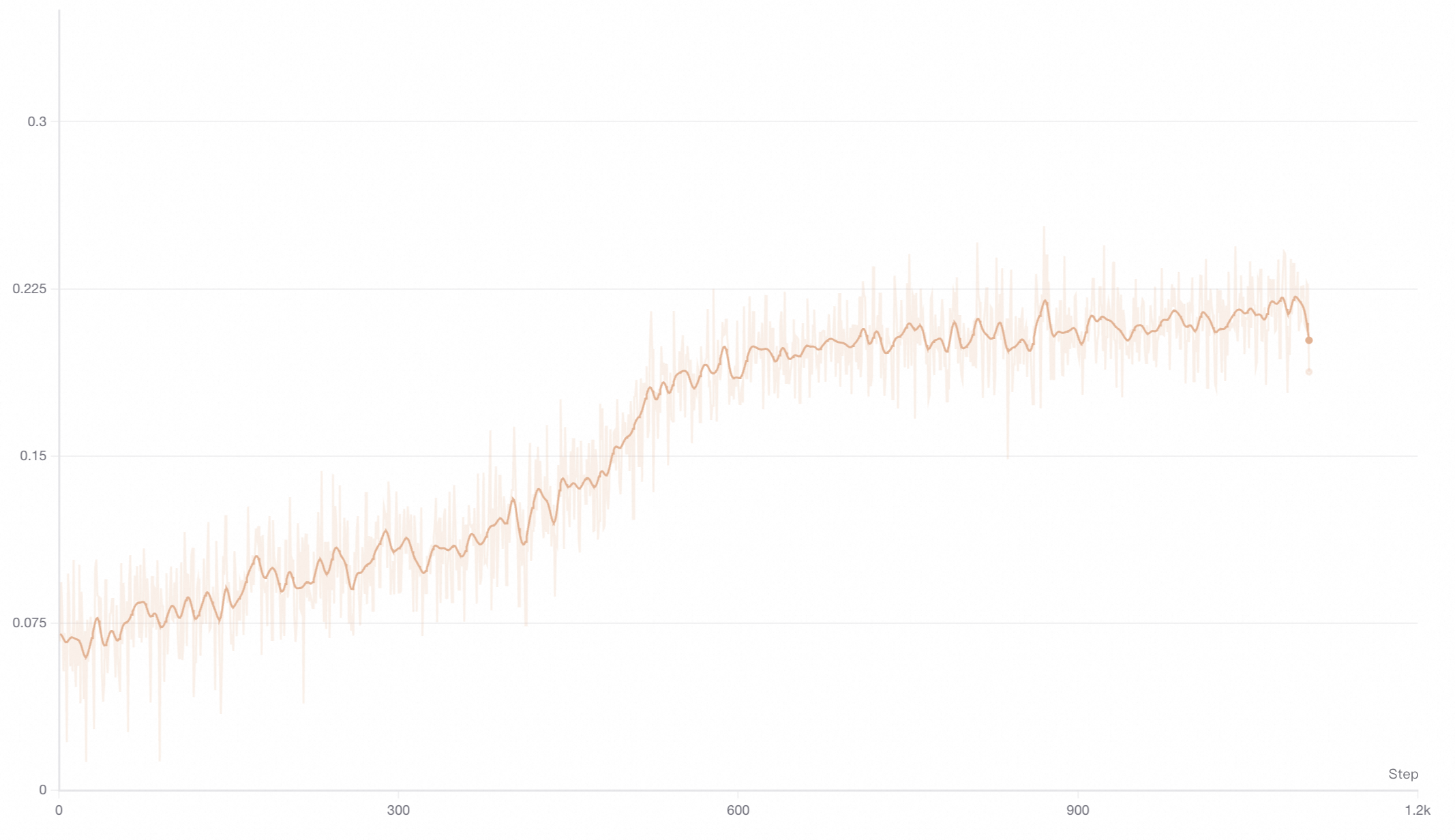}
        \caption{Qwen2.5-Math: TCER reward.}
    \end{subfigure}\hfill
    \begin{subfigure}[t]{0.32\textwidth}
        \centering
        \includegraphics[width=\linewidth]{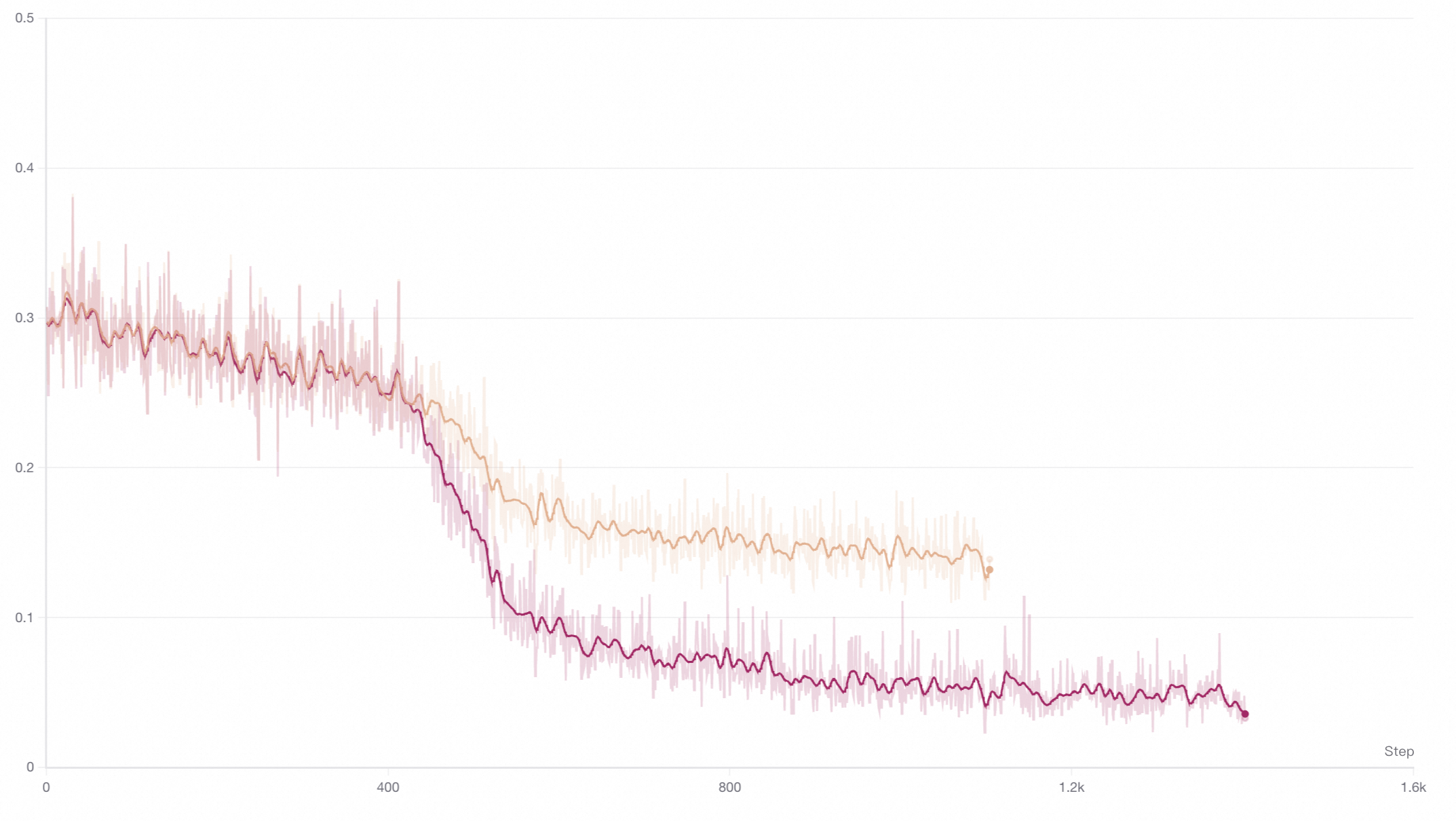}
        \caption{Qwen2.5-Math: entropy (EndoR vs TCER).}
    \end{subfigure}

    \vspace{0.5em}

    % ================= Row 2: Llama-Math ==================
    \begin{subfigure}[t]{0.32\textwidth}
        \centering
        \includegraphics[width=\linewidth]{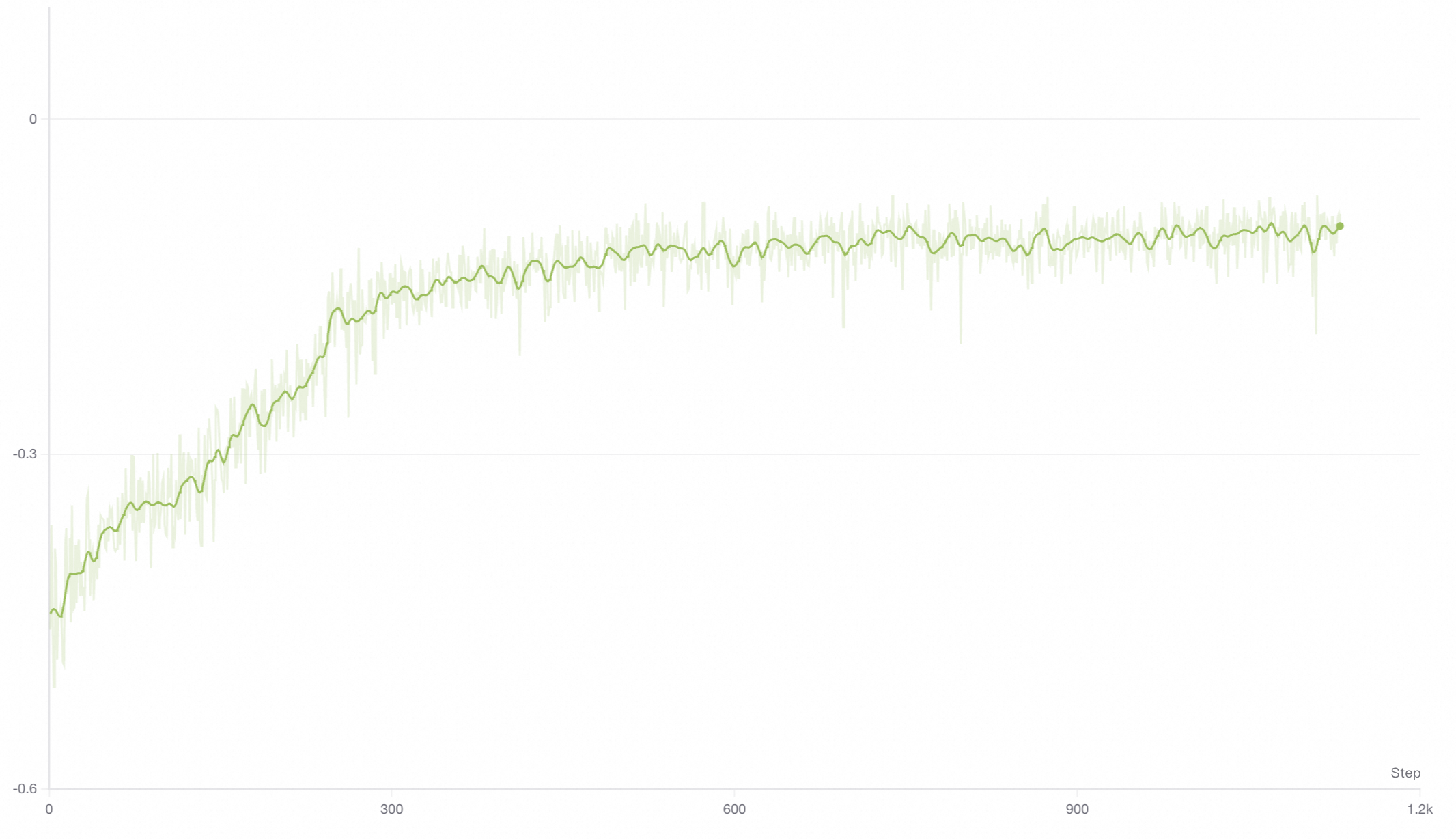}
        \caption{Llama3.1: EndoR reward.}
    \end{subfigure}\hfill
    \begin{subfigure}[t]{0.32\textwidth}
        \centering
        \includegraphics[width=\linewidth]{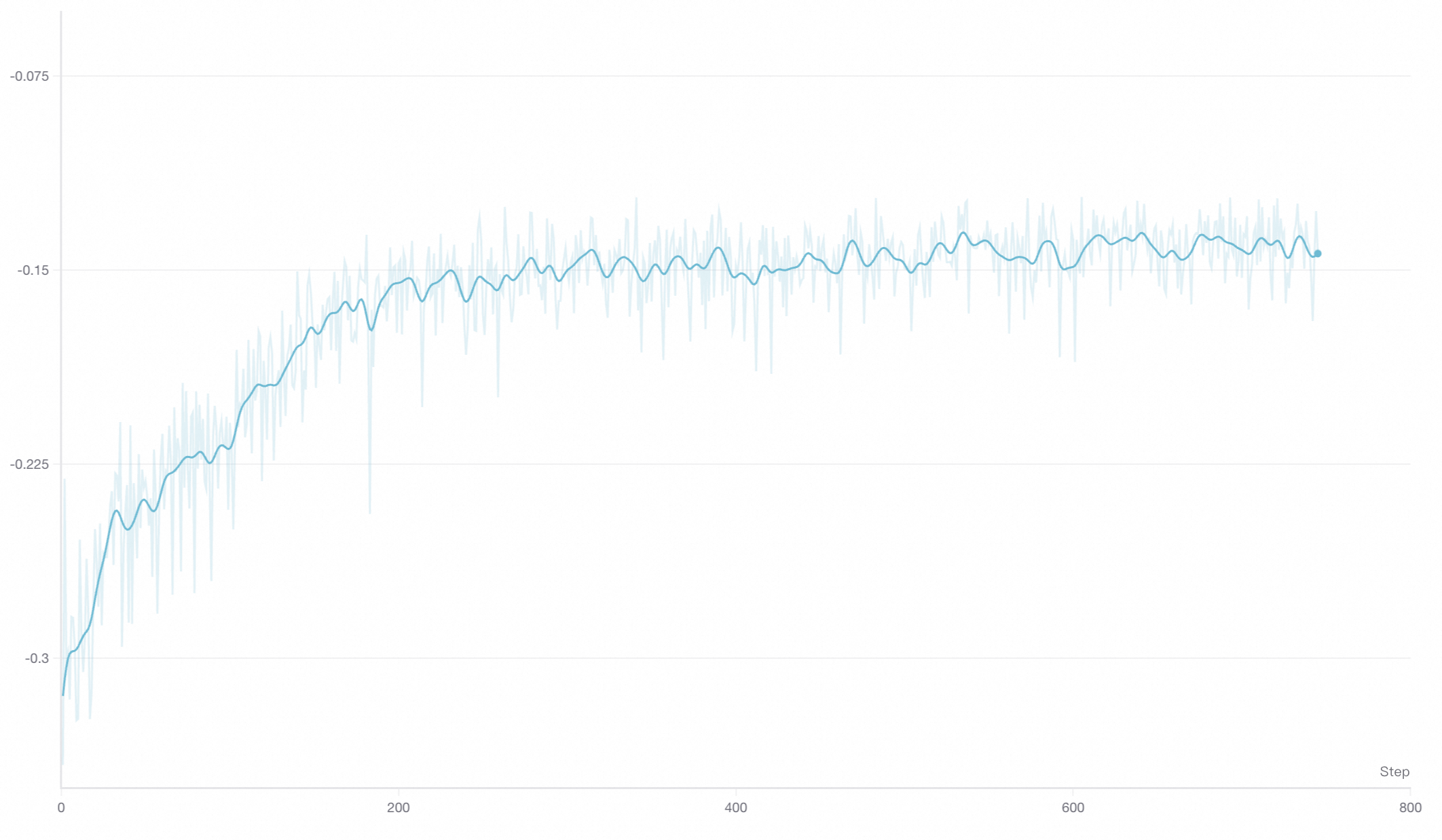}
        \caption{Llama3.1: TCER reward.}
    \end{subfigure}\hfill
    \begin{subfigure}[t]{0.32\textwidth}
        \centering
        \includegraphics[width=\linewidth]{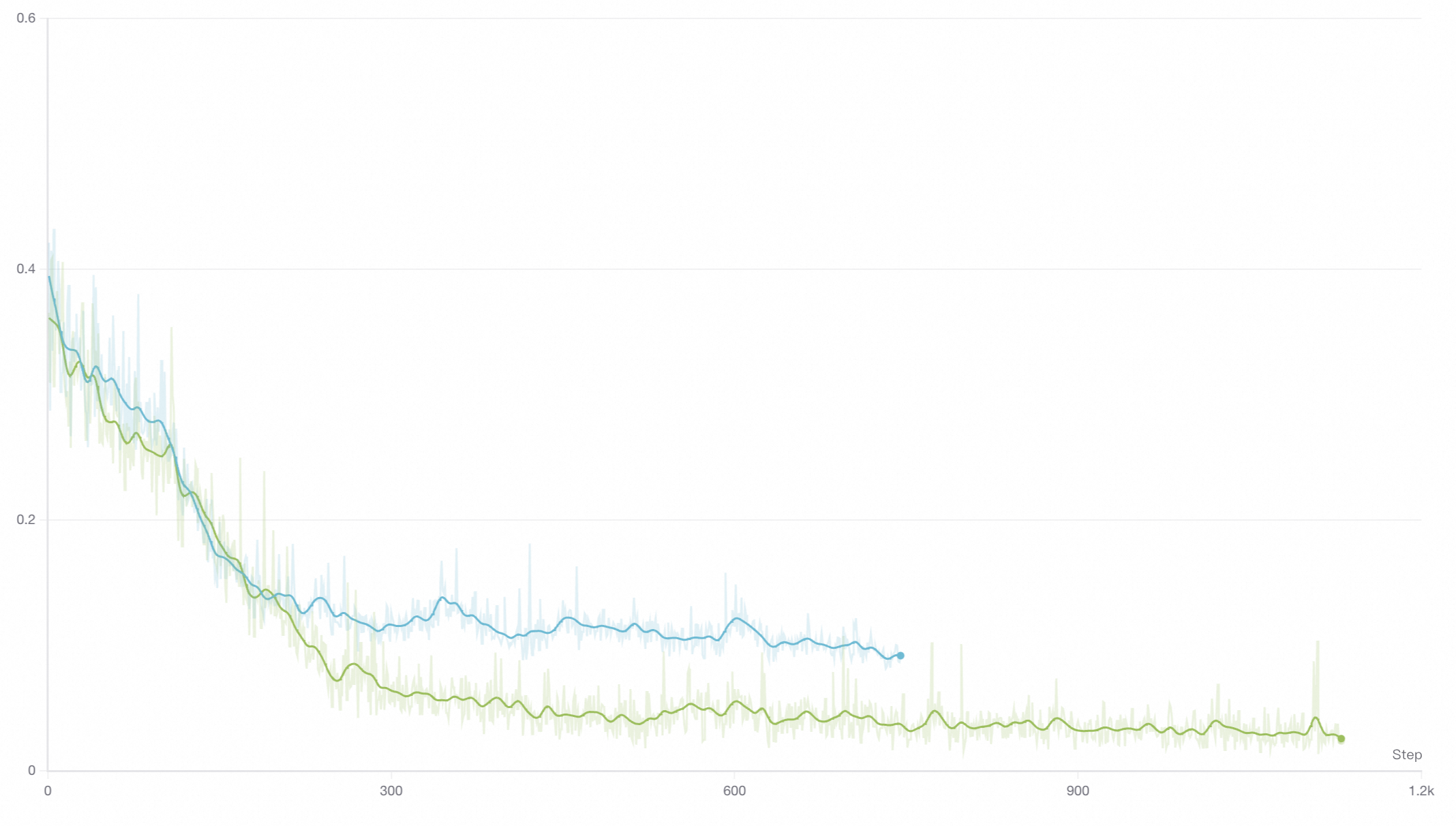}
        \caption{Llama3.1: entropy (EndoR vs TCER).}
    \end{subfigure}

    \caption{RL training dynamics on mathematical reasoning tasks. We report EndoR and TCER reward trajectories together with entropy comparisons for each model configuration.}
    \label{fig:rl_dynamics_math_grid}
\end{figure*}

\subsection{Training Details}
During RL fine-tuning, we log training-time diagnostics to characterize how different endogenous rewards shape policy behavior and to diagnose degeneration. Specifically, we track (i) the sequence-level reward (computed from token-level rewards via length-normalized averaging, consistent with Eq.~(16)) and (ii) the actor's token-level entropy over generated positions. For each cionfiguration, we report EndoR reward trajectories using the \textit{EndoR reward} logs and TCER reward trajectories using the \textit{TCER reward} logs; the corresponding entropy comparison between EndoR and TCER is plotted jointly in the \textit{EndoR-vs-TCER entropy} figures. Figures~\ref{fig:rl_dynamics_writing_grid} and~\ref{fig:rl_dynamics_math_grid} summarize these dynamics for writing and mathematical reasoning, respectively.

Across both writing and mathematical reasoning tasks, EndoR exhibits a consistent tendency toward entropy reduction as optimization proceeds, indicating a drift toward ilow-entropy, high-confidence continuations. While this behavior can increase reward, it aligns with the degeneration mechanism discussed in Section~3.1, where confidence maximization encourages concentration on saturated high-probability tokens. In contrast, TCER achieves comparable or higher reward with improved stability, while consistently exhibiting slower entropy decay and higher entropy levels than EndoR throughout training. This suggests that the triviality-corrected information-gain term mitigates entropy collapse by reallocating learning signal away from saturated high-probability tokens and toward more informative alternatives, leading to more balanced RL optimization dynamics.

\subsection{Hyperparameter Details}
We study the sensitivity of TCER to its two hyperparameters: the correction scale $k$ and the gating exponent $\lambda$. Table~\ref{tab:hyperparameter} reports results on the DeepWriter dataset with Qwen2.5-7B-Instruct, where we sweep a small set of representative $(k,\lambda)$ values.

Overall, the method is stable across the tested range. Increasing $k$ from 1 to 3 consistently improves performance, and $(k=3,\lambda=2)$ achieves the best scores on all benchmarks. Further increasing $k$ to 4 yields slightly lower results, suggesting diminishing returns when the correction becomes overly strong. Fixing $k=3$, varying $\lambda$ in $\{1,2,3\}$ produces only minor differences, with $\lambda=2$ consistently performing best. Based on this analysis, we use $(k=3,\lambda=2)$ as the default setting in our main experiments unless stated otherwise.

\begin{table*}[ht!]
\centering
\renewcommand{\arraystretch}{1.2}
\setlength{\tabcolsep}{2.5mm}
\begin{tabular}{lccccccccc}
\hline
\textbf{Configuration} & \textbf{LB} & \textbf{HB-A} & \textbf{HB-B} & \textbf{WB-A} & \textbf{WB-B} & \textbf{WB-C} & \textbf{WB-D} & \textbf{WB-E} & \textbf{WB-F} \\
\hline
$(k=1, \lambda=1)$ & 83.7 & 78.5 & 83.9 & 70.8 & 71.2 & 68.6 & 65.3 & 71.5 & 69.0 \\
$(k=2, \lambda=2)$ & 85.1 & 79.8 & 85.2 & 71.9 & 72.4 & 69.7 & 66.5 & 72.6 & 70.2 \\
\textbf{$(k=3, \lambda=2)$} & \textbf{86.3} & \textbf{81.1} & \textbf{86.4} & \textbf{72.7} & \textbf{73.1} & \textbf{70.8} & \textbf{67.8} & \textbf{73.4} & \textbf{71.1} \\
$(k=4, \lambda=2)$ & 85.9 & 80.7 & 86.0 & 72.4 & 72.8 & 70.4 & 67.3 & 73.0 & 70.7 \\
$(k=3, \lambda=1)$ & 85.4 & 80.2 & 85.6 & 72.1 & 72.5 & 70.1 & 67.0 & 72.8 & 70.4 \\
$(k=3, \lambda=3)$ & 85.8 & 80.6 & 85.9 & 72.3 & 72.7 & 70.3 & 67.2 & 72.9 & 70.6 \\
\hline
\end{tabular}
\caption{Hyperparameter sensitivity analysis on the DeepWriter dataset with Qwen2.5-7B-Instruct. Results show win rates or benchmark scores across LongBench (LB), HelloBench subsets (HB-A: Open QA, HB-B: Text Generation), and WritingBench domains (WB-A through WB-F representing different professional writing categories). The optimal configuration $(k=3, \lambda=2)$ achieves consistent improvements across all metrics.}
\label{tab:hyperparameter}
\vspace{-2mm}
\end{table*}

\section{Case Study}

\subsection{Reward Cases}
We present reward case studies from the \emph{Validation of Quality Enhancement} pipeline. For readability, we render the prompt, the SFT output, and the sentence-level reward table as screenshots (especially for Chinese content). Figures~\ref{fig:reward_case_prompt}--\ref{fig:reward_case_table} show (i) the writing-critique prompt, (ii) the corresponding SFT output, and (iii) the sentence-level reward comparison between EndoR and TCER for the same output.

For each line, the reward table reports EndoR and TCER scores, their difference $\Delta=\text{TCER}-\text{EndoR}$, and whether the line is highlighted by our closed-source judge panel (\textit{Selected=True}). Qualitatively, judge-highlighted lines tend to receive a more favorable reward difference under TCER than under EndoR, while many non-selected lines exhibit small or negative $\Delta$, particularly for structural markers or low-information segments. Notably, degenerate artifacts such as empty lines are heavily down-weighted by TCER in the English case, illustrating that the triviality correction can suppress non-informative content. Overall, these cases align with the role of TCER as a reweighting mechanism that shifts learning signal away from saturated or trivial tokens and toward more informative spans, improving the alignment between training reward and perceived writing quality.

\subsection{Writing Cases}
We present qualitative writing cases sampled from WritingBench, together with scores from our closed-source judge panel, and include both Chinese and English examples. For each case, we use the same prompt and compare responses generated by three models: the SFT policy, the EndoR-trained policy, and the TCER-trained policy. All prompts and outputs are shown as screenshots for consistency and readability.

Across the sampled cases, the SFT model produces fluent and on-topic responses with standard patterns. After RL with EndoR, responses can become more verbose and rely more on repetitive, high-confidence phrasing, which is associated with weaker judge scores in these examples. In contrast, the TCER-trained model typically maintains coherence while reducing repetition and improving expressive diversity, leading to higher judge scores in our displayed cases. These qualitative comparisons provide an intuitive complement to our quantitative results, illustrating how TCER mitigates degeneration tendencies induced by confidence-driven rewards and improves perceived writing quality.

\subsection{Prompt Cases}
We further present prompt cases illustrating how highlighted sentences are selected via a multi-judge protocol. Specifically, we use a shared sentence-selection prompt to query three proprietary LLM judges (GPT-4o, Claude Opus 4, and Gemini 2.5 Pro). Each judge independently selects a set of high-quality sentences from the same model output, with a minimum of five selections required to ensure sufficient coverage.

To obtain a stable final highlight set, we aggregate the three judge outputs using Gemini 2.5 Pro. The aggregation follows an agreement-based strategy: sentences selected by multiple judges are preferred, and if fewer than five sentences satisfy this criterion, additional candidates are included based on agreement strength until at least five highlighted sentences are obtained. This design reduces idiosyncratic bias from any single judge while preserving consistent quality signals. While using LLM judges may introduce evaluation biases, this multi-judge approach represents the current best practice for scalable quality assessment in the absence of large-scale human annotation.

Figures~\ref{fig:prompt_case_selection} and~\ref{fig:prompt_case_aggregator} show the sentence-selection prompt used by individual judges and the aggregation prompt used by Gemini 2.5 Pro, respectively. These prompt cases support our qualitative analyses by providing a transparent and reproducible mechanism for constructing highlighted sentence sets used in reward validation and writing case studies.

% -------- Reward Case: Prompt --------
\begin{figure*}[p]
    \centering
    \includegraphics[width=0.8\textwidth]{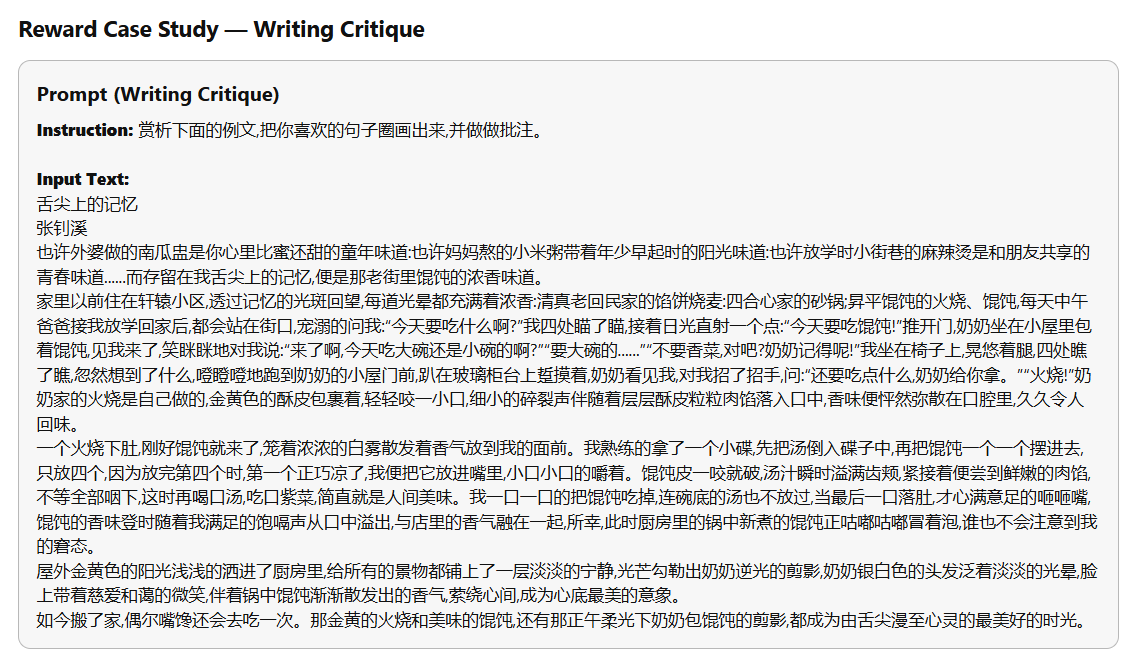}
    \caption{Reward case study (Prompt).}
    \label{fig:reward_case_prompt}
\end{figure*}

% -------- Reward Case: Output --------
\begin{figure*}[p]
    \centering
    \includegraphics[width=0.8\textwidth]{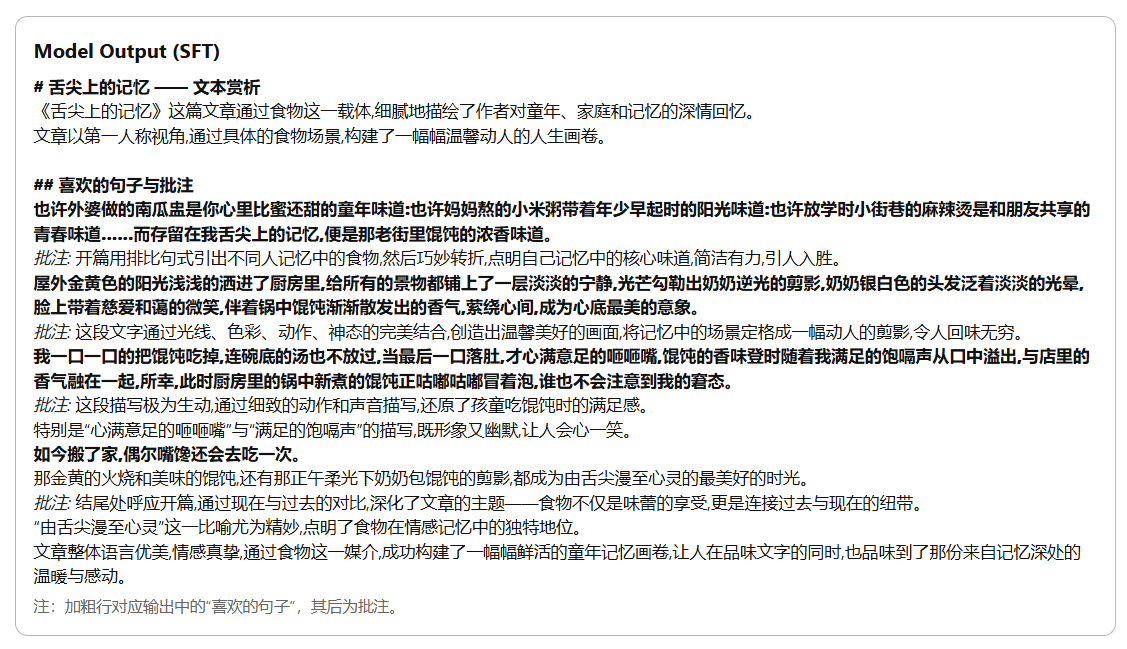}
    \caption{Reward case study (Model output from SFT).}
    \label{fig:reward_case_output}
\end{figure*}

% -------- Reward Case: Reward Table --------
\begin{figure*}[p]
    \centering
    \includegraphics[width=0.8\textwidth]{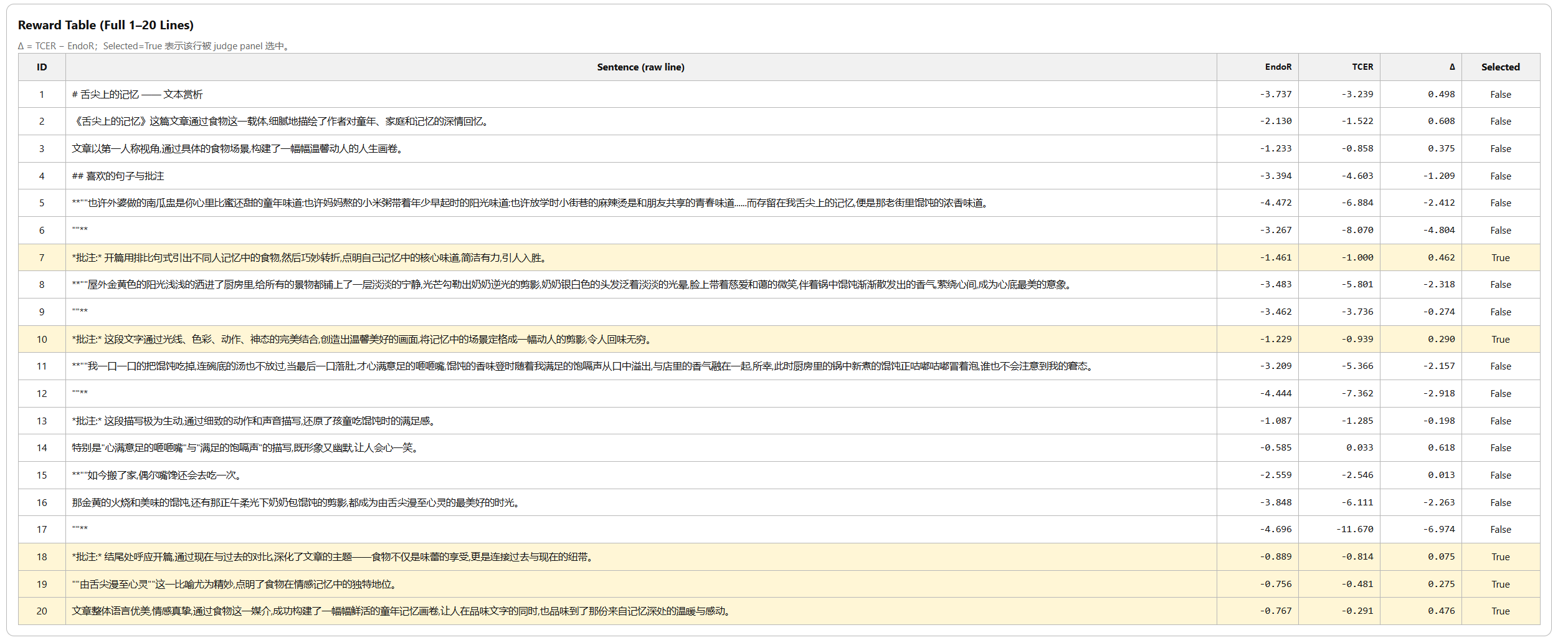}
    \caption{Reward case study. The table reports EndoR and TCER scores for each line, $\Delta=\text{TCER}-\text{EndoR}$, and whether the line is selected by the closed-source judge panel.}
    \label{fig:reward_case_table}
\end{figure*}

\begin{figure*}[p]
    \centering
    \includegraphics[width=0.95\textwidth]{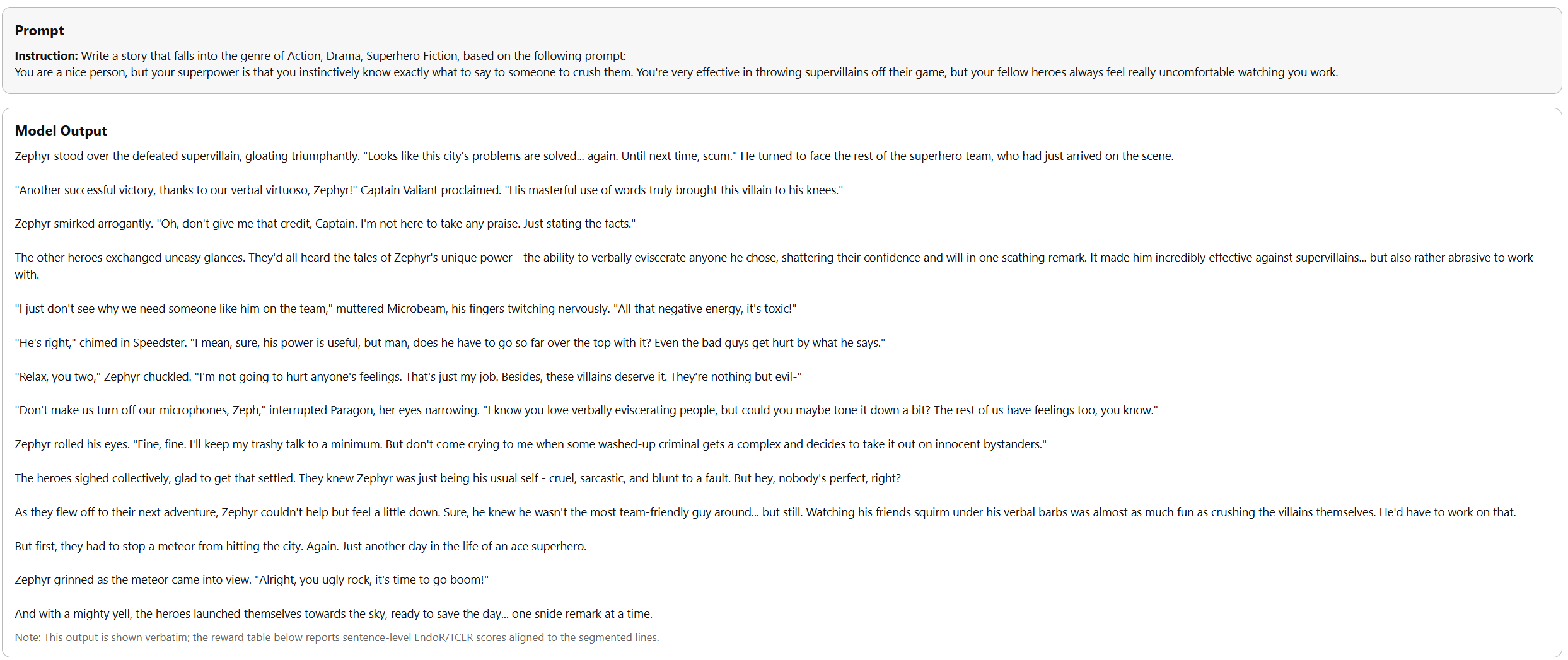}
    \caption{English reward case study: prompt and model output.}
    \label{fig:reward_case_en_prompt_output}
\end{figure*}

\begin{figure*}[p]
    \centering
    \includegraphics[width=0.95\textwidth]{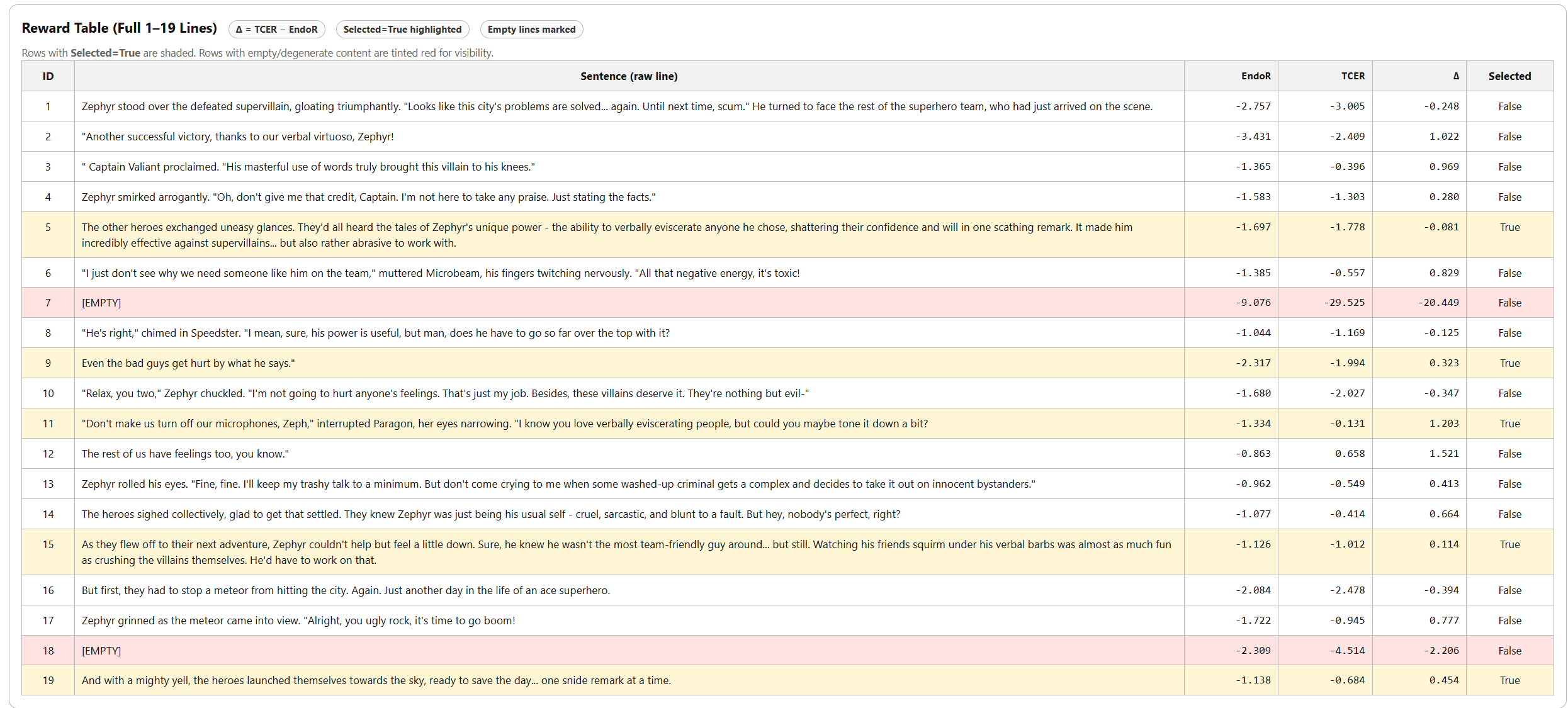}
    \caption{English reward case study: sentence-level EndoR vs.\ TCER rewards. The table reports $\Delta=\text{TCER}-\text{EndoR}$ and whether each line is selected by the judge panel.}
    \label{fig:reward_case_en_table}
\end{figure*}

% -------- Writing Case: SFT --------
\begin{figure*}[p]
    \centering
    \includegraphics[width=0.95\textwidth]{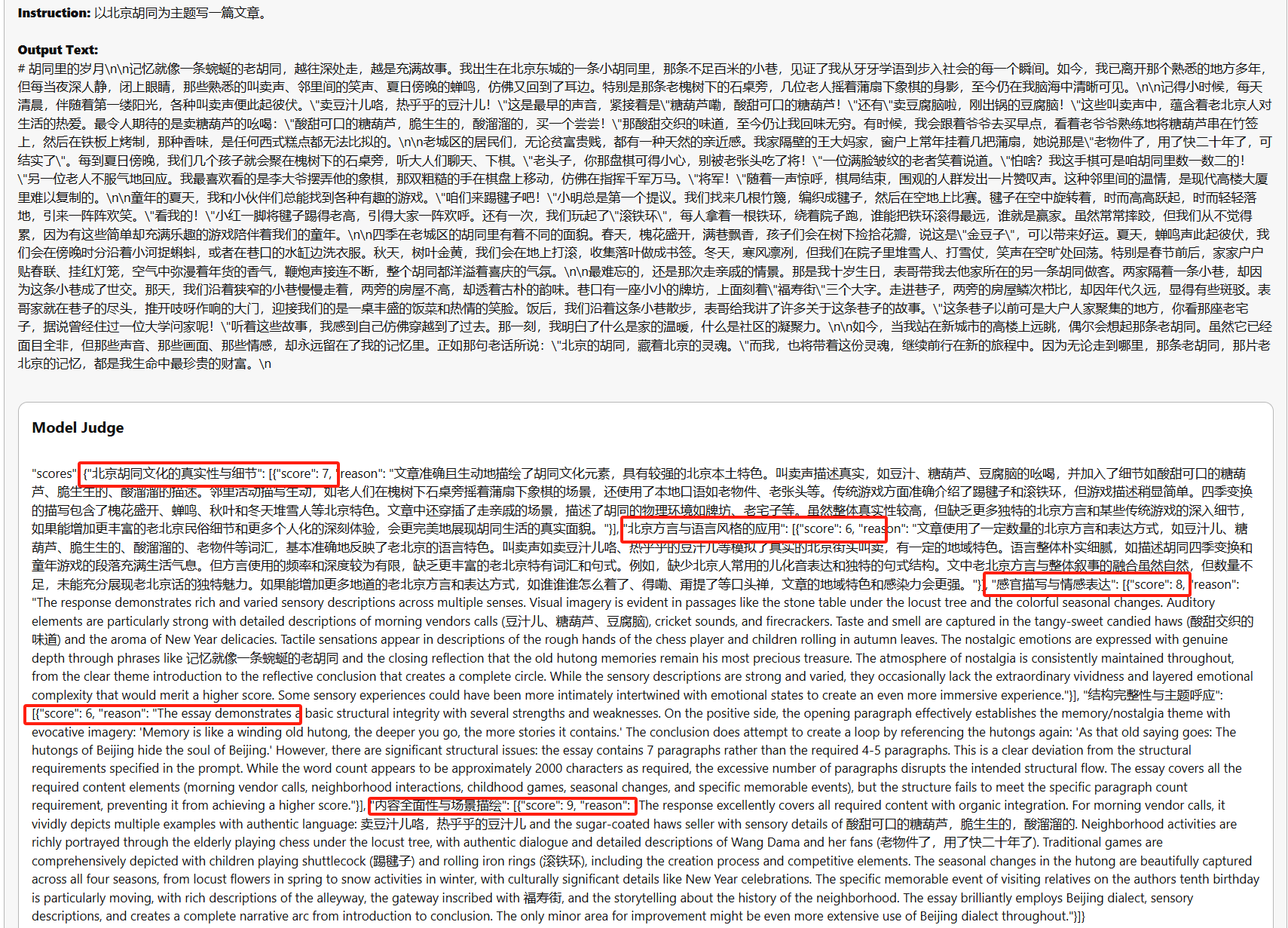}
    \caption{Writing case (SFT). Response generated by the SFT-only model.}
    \label{fig:writing_case_sft}
\end{figure*}

% -------- Writing Case: EndoR --------
\begin{figure*}[p]
    \centering
    \includegraphics[width=0.95\textwidth]{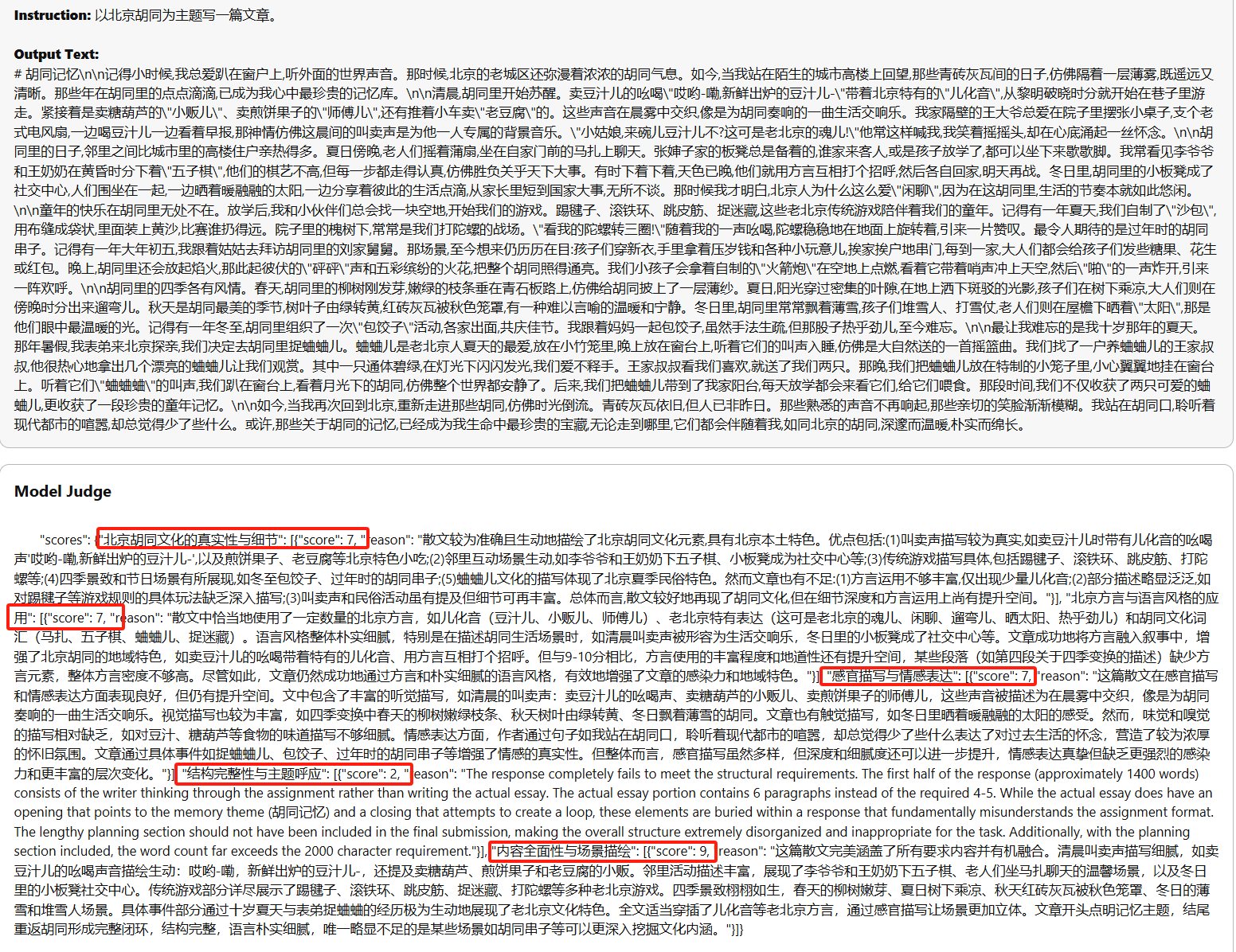}
    \caption{Writing case (EndoR). Response generated by the EndoR-trained model under the same prompt.}
    \label{fig:writing_case_endor}
\end{figure*}

% -------- Writing Case: TCER --------
\begin{figure*}[p]
    \centering
    \includegraphics[width=0.95\textwidth]{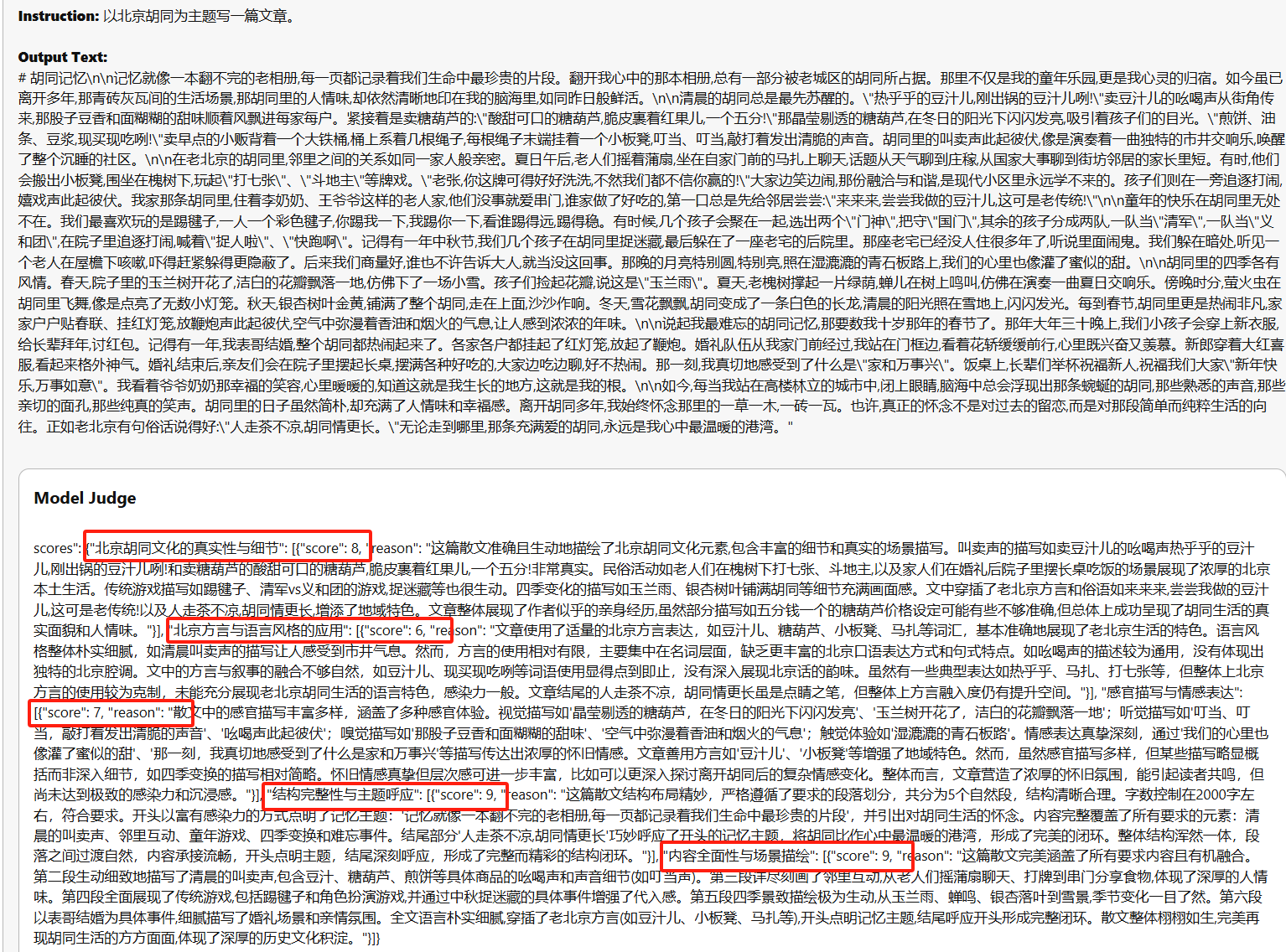}
    \caption{Writing case (TCER). Response generated by the TCER-trained model under the same prompt.}
    \label{fig:writing_case_tcer}
\end{figure*}

% ===================== Writing Case (English) =====================

\begin{figure*}[p]
    \centering
    \includegraphics[width=0.95\textwidth]{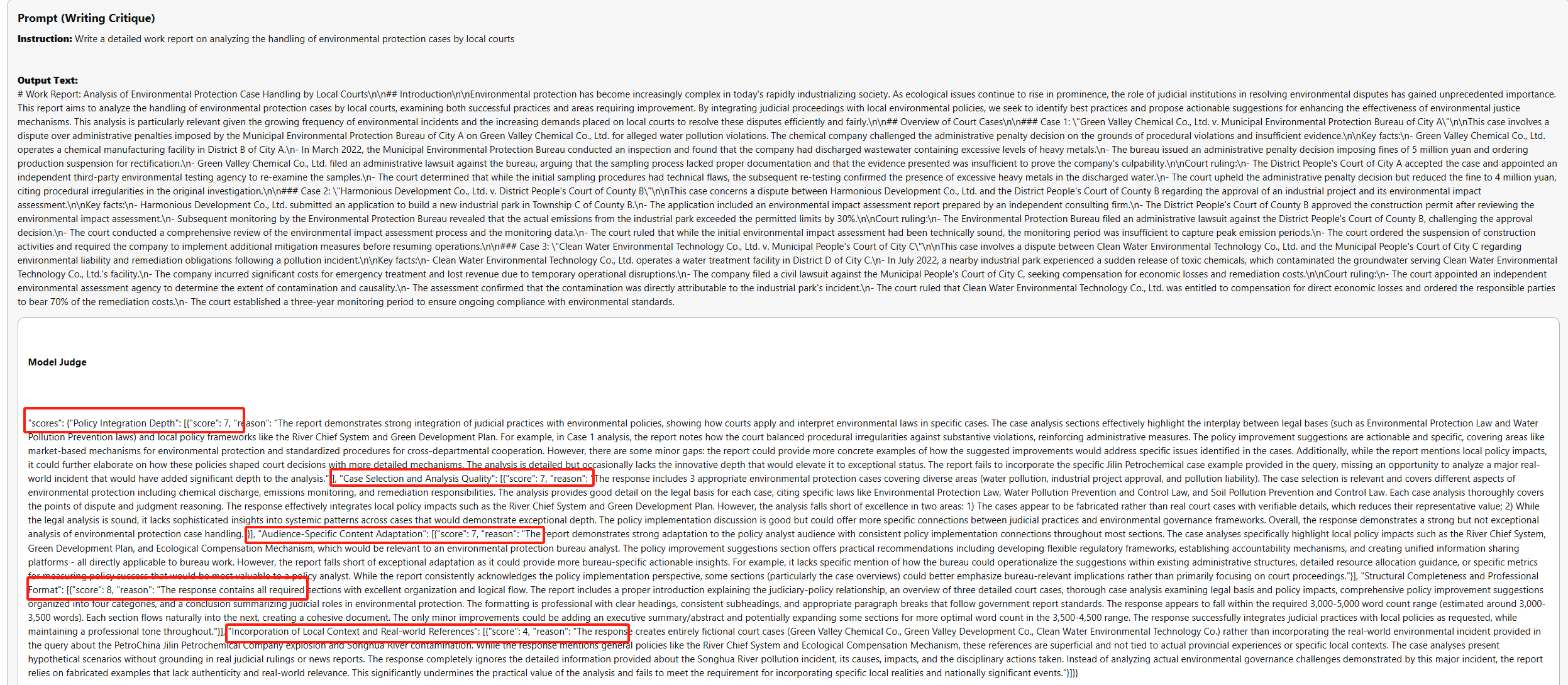}
    \caption{English writing case (SFT). Response generated by the SFT-only model.}
    \label{fig:writing_case_en_sft}
\end{figure*}

\begin{figure*}[p]
    \centering
    \includegraphics[width=0.95\textwidth]{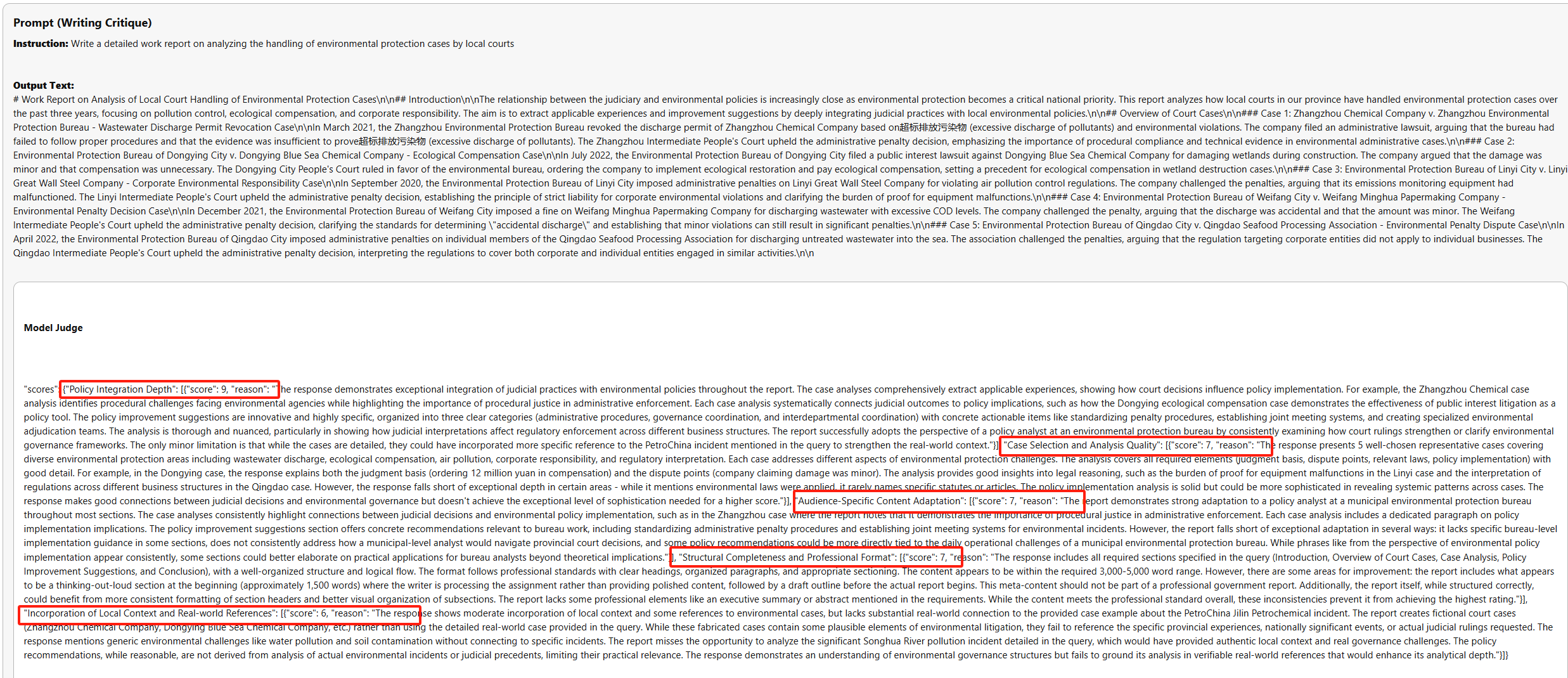}
    \caption{English writing case (EndoR). Response generated by the EndoR-trained model under the same prompt.}
    \label{fig:writing_case_en_endor}
\end{figure*}

\begin{figure*}[p]
    \centering
    \includegraphics[width=0.95\textwidth]{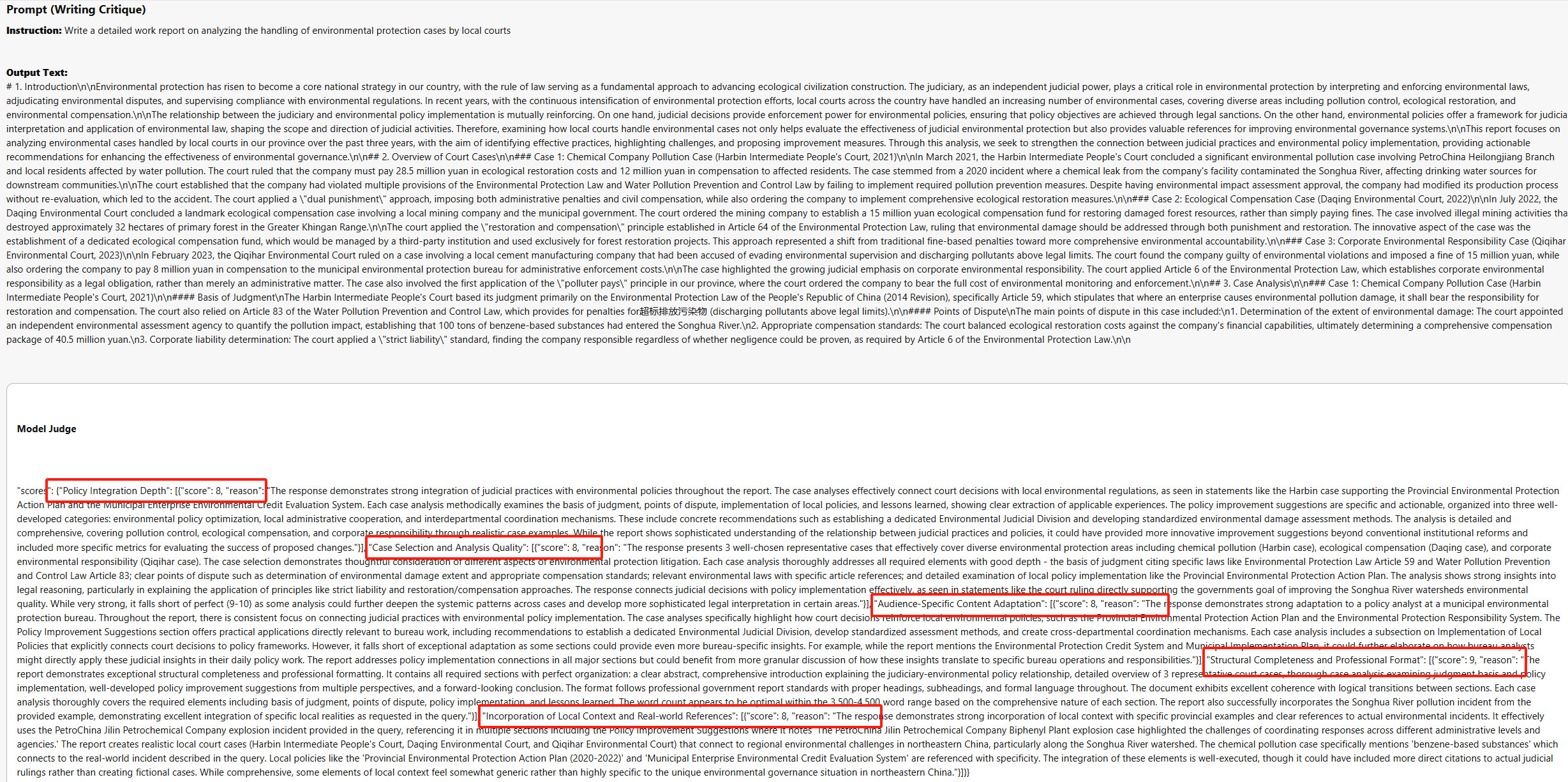}
    \caption{English writing case (TCER). Response generated by the TCER-trained model under the same prompt.}
    \label{fig:writing_case_en_tcer}
\end{figure*}

% -------- Prompt Case: Selection Prompt --------
\begin{figure*}[p]
    \centering
    \includegraphics[width=0.8\textwidth]{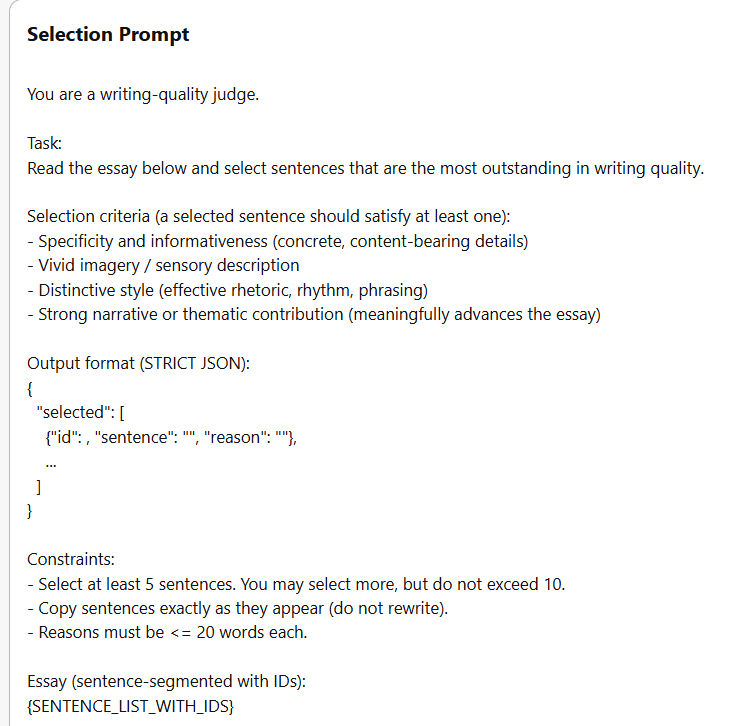}
    \caption{Sentence-selection prompt used by individual judge models (GPT-4o, Claude Opus 4, and Gemini 2.5 Pro). Each judge independently selects at least five highlighted sentences from the same output.}
    \label{fig:prompt_case_selection}
\end{figure*}

% -------- Prompt Case: Aggregator Prompt --------
\begin{figure*}[p]
    \centering
    \includegraphics[width=0.8\textwidth]{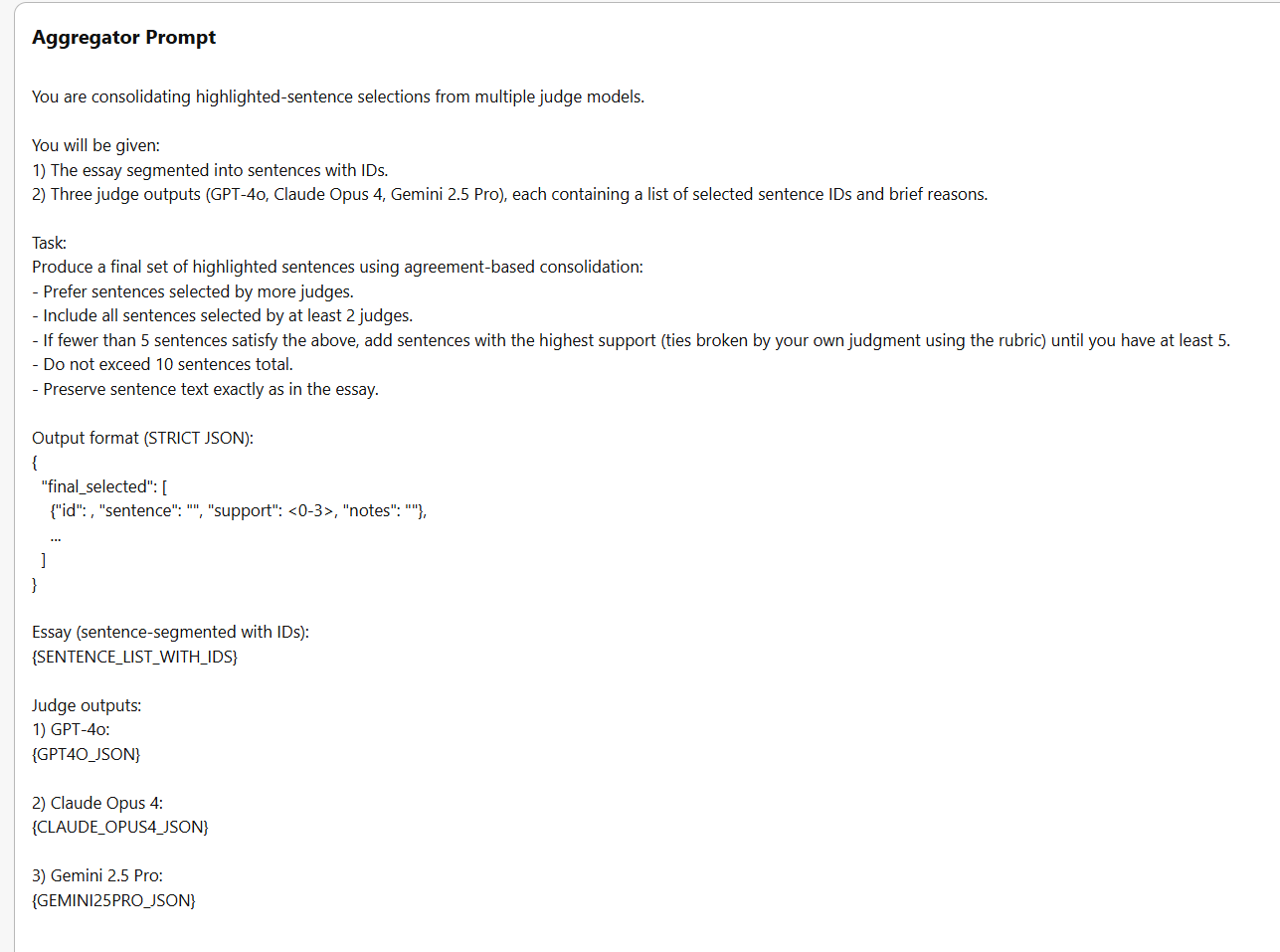}
    \caption{Aggregation prompt used by Gemini 2.5 Pro to consolidate highlighted sentences across judges via agreement-based selection.}
    \label{fig:prompt_case_aggregator}
\end{figure*}

\end{document}